\documentclass[aap]{stylefile}
\RequirePackage{amsthm,amsmath,amsfonts,amssymb}
\RequirePackage{natbib}
\RequirePackage{mathtools}
\RequirePackage[shortlabels]{enumitem}
\RequirePackage{algorithm}
\RequirePackage{algorithmic}
\RequirePackage{tikz}
\RequirePackage[colorlinks,citecolor=blue,urlcolor=blue]{hyperref}
\RequirePackage{graphicx}

\newcommand\blfootnote[1]{
  \begingroup
  \renewcommand\thefootnote{}\footnote{#1}
  \addtocounter{footnote}{-1}
  \endgroup
}

\usepackage{ragged2e}
\usepackage{bm}
\usepackage{multicol}
\usepackage{multirow}
\usepackage{pifont}
\usepackage{todonotes}

\startlocaldefs
\theoremstyle{plain}
\newtheorem{theorem}{Theorem}[section]
\newtheorem{lemma}{Lemma}[section]

\newtheorem{proposition}{Proposition}[section]

\theoremstyle{definition}
\newtheorem{definition}{Definition}[section]
\newtheorem{condition}{Condition}[section]

\newtheorem{assumption}{Assumption}[section]

\theoremstyle{remark}
\newtheorem*{remark}{Remark}

\endlocaldefs

\begin{document}

\begin{frontmatter}
\title{Stationary Behavior of Constant Stepsize SGD Type Algorithms: An Asymptotic Characterization}

\begin{aug}
\author[A]{\fnms{Zaiwei} \snm{Chen}\ead[label=e1,mark]{zchen458@gatech.edu}},
\author[A]{\fnms{Shancong} \snm{Mou}\ead[label=e2,mark]{shancong.mou@gatech.edu}},
\and
\author[A]{\fnms{Siva Theja} \snm{Maguluri}\ead[label=e3,mark]{siva.theja@gatech.edu}}

\address[A]{
Geogia Institute of Technology,
\printead{e1,e2,e3}}

\blfootnote{A version of this work was submitted to a conference on 28th May, 2021}
\blfootnote{Equal contribution between Zaiwei Chen and Shancong Mou}

\end{aug}

\begin{abstract}
Stochastic approximation (SA) and stochastic gradient descent (SGD) algorithms are work-horses for modern machine learning algorithms. 
Their constant stepsize variants are preferred in practice due to  fast convergence behavior. However, constant step stochastic iterative algorithms do not converge asymptotically to the optimal solution, but instead have a stationary distribution, which in general cannot be analytically characterized. In this work, we study the asymptotic behavior of the appropriately scaled stationary distribution, in the limit when the constant stepsize goes to zero. Specifically, we consider the following three settings:  (1) SGD algorithms with smooth and strongly convex objective, (2) linear SA algorithms involving a Hurwitz matrix, and (3) nonlinear SA algorithms involving a contractive operator. When the iterate is scaled by $1/\sqrt{\alpha}$, where $\alpha$ is the constant stepsize, we show that the limiting scaled stationary distribution is a solution of an integral equation. Under a uniqueness assumption (which can be removed in certain settings) on this equation, we further characterize the limiting distribution as a Gaussian distribution whose covariance matrix is the unique solution of a suitable Lyapunov equation. For SA algorithms beyond these cases, our numerical experiments suggest that unlike central limit theorem type results: (1) the scaling factor need not be $1/\sqrt{\alpha}$, and (2) the limiting distribution need not be Gaussian. Based on the numerical study, we come up with a formula to determine the right scaling factor, and make insightful connection to the Euler-Maruyama discretization scheme for approximating stochastic differential equations.
\end{abstract}

\end{frontmatter}

\section{Introduction}\label{sec:intro}

Stochastic approximation (SA) algorithms are the major work-horses for solving large-scale optimization and machine learning problems. Theoretically, to achieve asymptotic convergence, we should use diminishing stepsizes (learning rates) with proper decay rate \citep{nemirovski2009robust}. However, constant stepsize SA algorithms are preferred in practice due to their faster convergence \citep{goodfellow2016deep}. In that case, instead of converging asymptotically to the desired solution, the iterates of constant stepsize SA have a stationary distribution. Although in many cases such weak convergence to a stationary distribution was established in the literature, it is not possible to fully characterize the limiting distribution. The reason is that, when constant stepsize is used, the distribution of the noise sequence within the SA algorithm plays an important role in the stationary distribution of the iterates. Since the distribution of the noise is in general unknown, the stationary distribution cannot be analytically characterized. In this work, building upon the works on stationary distribution of constant stepsize SA algorithms, we aim at understanding the limiting behavior of the properly scaled stationary distribution as the constant stepsize goes to zero.

More formally, with initialization $X_0^{(\alpha)}\in\mathbb{R}^d$, consider the SA algorithm

\begin{align}\label{algo:sa}
	X_{k+1}^{(\alpha)}=X_k^{(\alpha)}+\alpha\left(F(X_k^{(\alpha)})+w_k\right),
\end{align}
where $F:\mathbb{R}^d\mapsto\mathbb{R}^d$ is a general nonlinear operator, $\alpha$ is the constant stepsize, and $\{w_k\}$ is the noise sequence. Observe that SA algorithm (\ref{algo:sa}) can be viewed as an iterative algorithm for solving the equation $F(x)=0$ in the presence of noise \citep{robbins1951stochastic}. A typical example is when $F(x)=-c \nabla f(x)$ (where $c>0$ is a constant) for some objective function $f(\cdot)$, in this case Algorithm (\ref{algo:sa}) becomes the popular stochastic gradient descent (SGD) algorithm for minimizing $f(\cdot)$ \citep{lan2020first,bottou2018optimization}. Another example lies in the context of reinforcement learning, where $F(x)=\mathcal{T}(x)-x$, and $\mathcal{T}(\cdot)$ is the Bellman operator \citep{sutton2018reinforcement}. In this case, Algorithm (\ref{algo:sa}) captures popular reinforcement learning algorithms such as TD-learning \citep{sutton1988learning} and $Q$-learning \citep{watkins1992q}.

Under some mild conditions on the operator $F(\cdot)$, it was shown in the literature that the sequence $\{X_k^{(\alpha)}\}$ converges weakly to some random variable $X^{(\alpha)}$ \citep{dieuleveut2017bridging, bianchi2020convergence,yu2020analysis, durmus2021riemannian}. However, for a fixed $\alpha$, it is not possible to fully characterize the distribution of $X^{(\alpha)}$ because it depends on the distribution of the noise sequence $\{w_k\}$, which is usually unknown. In this work, we further consider letting $\alpha$ go to zero, and study the distribution of a properly centered and scaled iterate. Specifically, let  $Y_k^{(\alpha)}:=(X_k^{(\alpha)}-x^*)/g(\alpha)$, where $x^*$ is the solution of $F(x)=0$, and $g:\mathbb{R}\mapsto\mathbb{R}$ is a properly chosen scaling function\footnote{The scaling function is unique up to a numerical factor}. When $k$ goes to infinity, we expect that $Y_k^{(\alpha)}$ converges weakly to some random variable $Y^{(\alpha)}$. Then we let $\alpha$ go to zero, and our goal is to further characterize the weak limit of $Y^{(\alpha)}$. Notice that proper scaling of the iterates is essential for raveling its fine grade behavior because otherwise the limiting distribution of the un-scaled iterates will converge to a singleton as the stepsize $\alpha$ 
goes to zero, which is analogous to the almost sure convergence result for using diminishing stepsizes in SA literature.

To summarize, we want to find a suitable scaling function $g(\cdot)$ and characterize the following two-step weak convergence of the centered scaled iterate $Y^{(\alpha)}_k=(X^{(\alpha)}_k-x^*)/g(\alpha)$:
\begin{align}\label{eq:two_limit}
	Y_k^{(\alpha)}\overset{k\rightarrow\infty}{\Longrightarrow} Y^{(\alpha)}\overset{\alpha\rightarrow 0}{\Longrightarrow}Y,
\end{align}
where we use the notation $\Rightarrow$ for weak convergence (or convergence in distribution).

\subsection{Main Contributions}\label{subsec:contribution}
Our main contributions are twofold.

\textit{Characterizing the Distribution of $Y$.}  
We propose a general framework for characterizing the distribution of $Y$ in the following 3 cases:
(1) SGD with a smooth and strongly convex objective, (2) linear SA with a Hurwitz matrix, and (3) SA involving a contractive operator. In particular, we show that in all three cases above the correct scaling function is $g(\alpha)=\sqrt{\alpha}$, and the distribution of $Y$ is Gaussian with mean zero and covariance matrix being the unique solution of an appropriate Lyapunov equation. Our proof is to use the characteristic function as a test function to obtain an integral equation of the distribution of $Y$, and then show that the desired Gaussian distribution solves the integral equation.

\textit{Determining the Suitable Scaling Function.} For more general SA algorithms, we show empirically that the scaling function need not be $g(\alpha)=\sqrt{\alpha}$ and the distribution of $Y$ need not be Gaussian. Inspired by this observation, we propose a method to find the the correct scaling function for general SA algorithms. In particular, our results indicate that the scaling function $g(\alpha)$ should be chosen such that (1) $\lim_{\alpha\rightarrow 0}\frac{\alpha}{g(\alpha)}=0$, and (2) the function $\Tilde{F}(\cdot)$ defined by $\Tilde{F}(y)=\lim_{\alpha\rightarrow 0}\frac{g(\alpha)F(yg(\alpha)+x^*)}{\alpha}$ is non-trivial in the sense that it is not identically zero or infinity. Our proposed condition is verified in numerical experiments. Moreover, we make an insightful connection between the choice of $g(\alpha)$ and the Euler-Maruyama discretization scheme for approximating a stochastic differential equation (SDE) -- Langevin diffusion \citep{sauer2012numerical}.

\subsection{An Illustrative Example}\label{subsec:illustrative}

In this section, we provide an example to further illustrate the problem we are going to study.
Consider SA algorithm (\ref{algo:sa}). Suppose that $F(x)=-x$ is a scalar valued function, and $\{w_k\}$ is a sequence of i.i.d. standard normal random variables. We make such noise assumption here only for ease of exposition, and it will be relaxed in later sections.  In this case, the SA algorithm (\ref{algo:sa}) becomes 
\begin{align}\label{eq:sa1}
	X_{k+1}^{(\alpha)}=(1-\alpha)X_k^{(\alpha)}+\alpha w_k.
\end{align}
This algorithm has the following two simple interpretations: (1) it can be viewed as the SGD algorithm for minimizing a quadratic objective function $f(x)=x^2/2$, which has a unique minimizer $x^*=0$, (2) it can also be viewed as an SA algorithm for solving the fixed-point equation $\mathcal{T}(x)=x$ with $\mathcal{T}(x)$ being identically equal to zero, therefore $x^*=0$ is the unique fixed-point.

Since $x^*=0$, let $Y_k^{(\alpha)}=X_k^{(\alpha)}/\sqrt{\alpha}$ be the centered scaled iterate. To obtain an update equation for $Y_k^{(\alpha)}$, dividing both sides of Eq. (\ref{eq:sa1}) by $\sqrt{\alpha}$ and we have for all $k\geq 0$:
\begin{align*}
	Y_k^{(\alpha)}&=(1-\alpha)Y_{k-1}^{(\alpha)}+\sqrt{\alpha} w_{k-1}\\
	&=(1-\alpha)^2Y_{k-2}^{(\alpha)}+(1-\alpha)\sqrt{\alpha} w_{k-2}+\sqrt{\alpha} w_{k-1}\\
	&=\cdots\\
	&=(1-\alpha)^kY_0^{(\alpha)}+\sum_{i=0}^{k-1}(1-\alpha)^{k-i-1}\sqrt{\alpha} w_i.
\end{align*}
Since $Y_k^{(\alpha)}$ is a linear combination of independent Gaussian random variables, $Y_k^{(\alpha)}$ itself is also Gaussian. Since
\begin{align*}
	\mathbb{E}[Y_k^{(\alpha)}]&=(1-\alpha)^kY_0^{(\alpha)}+\sum_{i=0}^{k-1}(1-\alpha)^{k-i-1}\sqrt{\alpha} \mathbb{E}[w_i]=(1-\alpha)^kY_0^{(\alpha)},
\end{align*}
and
\begin{align*}
	\mathbb{V}[Y_k^{(\alpha)}]&=\mathbb{V}\left[(1-\alpha)^kY_0^{(\alpha)}+\sum_{i=0}^{k-1}(1-\alpha)^{k-i-1}\sqrt{\alpha} w_i\right]=\frac{1}{2-\alpha}\left(1-(1-\alpha)^{2k}\right),
\end{align*}
where $\mathbb{E}[\;\cdot\;]$ and  $\mathbb{V}[\;\cdot\;]$ denote the mean and variance, respectively, 
we have $\lim_{k\rightarrow \infty}\mathbb{E}[Y_k^{(\alpha)}]=0$ and $\lim_{k\rightarrow \infty}\mathbb{V}[Y_k^{(\alpha)}]=\frac{1}{2-\alpha}$. Therefore, the sequence $Y_k^{(\alpha)}$ converges weakly to a random variable $Y^{(\alpha)}$, whose distribution is $ \mathcal{N}(0,\frac{1}{2-\alpha})$. In this case, we are able to analytically characterize the distribution of $Y^{(\alpha)}$ for a fixed $\alpha$ because of the simplicity of the algorithm (\ref{eq:sa1}) and the noise $\{w_k\}$ being i.i.d. Gaussian. For the general SA algorithm (\ref{algo:sa}) with limited information on the noise sequence $\{w_k\}$, it is in general not possible to fully characterize the distribution of $Y^{(\alpha)}$.

Now consider the second weak convergence in Eq. (\ref{eq:two_limit}). Since we have already shown that $Y^{(\alpha)}\sim \mathcal{N}(0,\frac{1}{2-\alpha})$. As $\alpha$ goes to zero, we see that $Y^{(\alpha)}$ converges weakly to a random variable $Y$, whose distribution is $\mathcal{N}(0,\frac{1}{2})$. As opposed to the first weak convergence in Eq. (\ref{eq:two_limit}), where the distribution of $Y^{(\alpha)}$ in general cannot be fully characterized, we are able to characterize (in later sections) the distribution of $Y$ for the more general algorithm (\ref{algo:sa}), and for more general noise assumptions. Intuitively, the reason is that as the constant stepsize decreases, the 
effect of the entire distribution of the noise $\{w_k\}$ on the distribution of $Y^\alpha$ is weakened. In the limit, only the mean and the variance of $w_k$ play roles in determining the distribution of $Y$. This is analogous to central limit theorem type of results.

To summarize, we have shown in the special case of Algorithm (\ref{eq:sa1}) that the correct scaling function is $g(\alpha)=\sqrt{\alpha}$, and the distribution of the limiting random variable $Y$ is a Gaussian distribution with mean zero and variance $1/\sqrt{2}$. In Section \ref{sec:distribution}, we extend this result to the more general algorithm (\ref{algo:sa}) with weaker noise assumptions.

\subsection{Related Literature}\label{subsec:literature}

Since proposed in \citep{robbins1951stochastic}, SA has been popular for solving large scale optimization in modern machine learning applications. Although require using diminishing stepsizes to achieve convergence \citep{hu2017diffusion, xie2020linear, mertikopoulos2020almost,shamir2013stochastic, li2019convergence,fehrman2020convergence, gower2021sgd}, constant stepsize is used in practice \citep{goodfellow2016deep}. In contrast to the success in machine learning practice, there is little discussion about the stationary distribution of constant step size SGD \citep{dieuleveut2017bridging, bianchi2020convergence,yu2020analysis,   durmus2021riemannian}. Among them, \cite{dieuleveut2017bridging} bridges Markov chain theory and the constant step size SGD algorithm. They provided an explicit asymptotic expansion of the moments of the averaged SGD iterates. \cite{bianchi2020convergence} studied the asymptotic behavior of constant step size SGD for a nonconvex nonsmooth, locally Lipchitz objective function. It was shown that in a small step size regime, the interpolated trajectory of the algorithm converges in probability towards the solutions of the differential inclusion $\dot{x}=\partial F(x)$ and the invariant distribution of the corresponding Markov chain converges weakly to the set of invariant distribution of the differential inclusion. \cite{yu2020analysis} established an asymptotic normality result for the constant step size SGD algorithm for a non-convex and non-smooth objective function satisfying a dissipativity property. \cite{durmus2021riemannian} first established non-asymptotic performance bounds under Lyapunov conditions and then proved that for any step size, the corresponding Markov chain admits a unique stationary distribution. 

The work mentioned before established the stationary distribution for almost strongly convex and smooth functions for a fixed constant stepsize.  Since the SGD iterates will converge to a singleton as the constant step size goes to zero, none of the previously mentioned literates can be applied to study the limiting behavior of SGD in this regime. To understand such behavior, we propose to study the properly centered and scaled iterates. Although not directly related, it shares a similar flavor when studying the limiting behavior of the stationary distribution of the stochastic gradient Langevin dynamics (SGLD) iterates as step size goes to zero. 

Another set of related literature is on the diffusion approximation of SGD \citep{li2017stochastic,feng2017semi,yang2021fast,sirignano2020stochastic, latz2021analysis}. Authors aim to approximate the trajectory of SGD by a diffusion process which solves an SDE. Notice that they also study the scaled version of the diffusion limit of SGD. However, different from our approach, their scale is in temporal domain and cannot be applied in our research.

The Markov chain perspective of studying SGD iterates when step size goes to zero \citep{dieuleveut2017bridging} is related to the heavy traffic analysis in queuing theory \citep{eryilmaz2012asymptotically}. It has been studied in the literature using fluid and diffusion limits 
\citep{gamarnik2006validity,harrison1988brownian,har_state_space,harlop_state_space,stolyar2004maxweight,Williams_state_space} where the interchange of limit is usually problematic \citep{eryilmaz2012asymptotically}. An alternative approach in stochastic networks is based on drift arguments introduced by \citep{eryilmaz2012asymptotically} and further generalized by \citep{maguluri2016heavy,maguluri2018optimal,hurtado2020transform,hurtado2020logarithmic, mou2020heavy}. We adopt similar techniques in quantifying the limiting distribution of the scaled SGD iterates. Notice that in stochastic networks, people mainly focus on finite state space Markov chains. However, when it comes to SGD iterates, the state space is continuous and thus more challenging.

The rest of this paper is organized as follows. In Section \ref{sec:distribution}, we characterize the distribution of $Y$ (cf. Eq (\ref{eq:two_limit})) in the following cases: (1) $F(\cdot)$ is the negative gradient for some smooth and strongly convex function $f(\cdot)$, (2) $F(x)=Ax+b$, where $A$ is a Hurwitz matrix, and (3) $F(x)=\mathcal{T}(x)-x$, where $\mathcal{T}(\cdot)$ is a contraction operator. In all three cases above, the scaling function is chosen to be $g(\alpha)=\sqrt{\alpha}$. Then in Section \ref{sec:scaling}, we first empirically show that for more general SA algorithms beyond these cases, the scaling function need not be $g(\alpha)=\sqrt{\alpha}$ and the distribution of $Y$ need not be Gaussian. Then we present a method to determine the scaling function for more general SA algorithms and make connection to the Euler-Maruyama discretization scheme for approximating SDE. Finally, we conclude this paper in Section \ref{sec:conclusion}.

\section{Characterizing the Asymptotic Stationary Distribution}\label{sec:distribution}

Through out this section, we make the following assumption regarding the noise sequence $\{w_k\}$.
\begin{assumption}\label{as:noise}
	The sequence $\{w_k\}$ is independent and identically distributed with mean zero and a positive definite covariance matrix $\Sigma\in\mathbb{R}^{d\times d}$.
\end{assumption}

Note that Assumption \ref{as:noise} is much weaker than the assumption used in Section \ref{subsec:illustrative}, where the noise is assumed to obey the Gaussian distribution. Nevertheless, extending our results to the more general noise setting (e.g. martingale difference noise, Markovian noise, etc) is one of our future directions.

\subsection{SGD for Minimizing a Smooth and Strongly Convex Objective}\label{subsec:SGD}

Suppose that $F(x)=-\nabla f(x)$, where $f(\cdot)$ is an objective function. Then the SA algorithm becomes
\begin{align}\label{algo:SGD}
	X_{k+1}^{(\alpha)}=X_k^{(\alpha)}+\alpha\left(-\nabla  f(X_k^{(\alpha)})+w_k\right),
\end{align}
which is the well-known SGD algorithm for minimizing $f(\cdot)$. To proceed, we require the following definition.
\begin{definition}
	A convex differentiable function $h:\mathbb{R}^d\mapsto\mathbb{R}$ is said to be $L$-smooth and $\sigma$-strongly convex with respect to the Euclidean norm $\|\cdot\|_2$ if and only if
	\begin{align*}
		h(y)&\leq h(x)+\langle \nabla h(x),y-x\rangle+\frac{L}{2}\|x-y\|_2^2\tag{$L$-smooth},\\
		h(y)&\geq  h(x)+\langle \nabla h(x),y-x\rangle+\frac{\sigma}{2}\|x-y\|_2^2\tag{$\sigma$-convex}
	\end{align*}
	for all $x,y\in\mathbb{R}^d$.
\end{definition}

To characterize the asymptotic behavior of Algorithm (\ref{algo:SGD}), we make the following assumption.

\begin{assumption}\label{as:SGD}
	The function $f:\mathbb{R}^d\mapsto\mathbb{R}$ is twice differentiable, and is $L$-smooth and $\sigma$-strongly convex.
\end{assumption}

Under Assumption \ref{as:SGD}, the function $f(x)$ has a unique minimizer (or $F(x)=0$ has a unique solution), which we have denoted by $x^*$. To proceed, let $Y_k^{(\alpha)}=(X_k^{(\alpha)}-x^*)/\sqrt{\alpha}$ be the centered scaled iterate. We first write down the corresponding update equation of $Y_k^{(\alpha)}$ in following:
\begin{align}\label{algo:Y:SGD}
	Y_{k+1}^{(\alpha)}=Y_k^{(\alpha)}-\sqrt{\alpha}\nabla f\left(\sqrt{\alpha} Y_k^{(\alpha)}+x^*\right)+\sqrt{\alpha} w_k,
\end{align}
which is obtained by first subtracting both sides of Eq. (\ref{algo:SGD}) by $x^*$, and then dividing by $\sqrt{\alpha}$.

We next characterize the two-step weak convergence (cf. Eq. (\ref{eq:two_limit})) in the following theorem. Let $H_f\in\mathbb{R}^{d\times d}$ be the Hessian matrix of the objective function $f(\cdot)$ evaluated at $x^*$, which is well-defined because $f(\cdot)$ is twice 
differentiable (cf. Assumption \ref{as:SGD}).

\begin{theorem} \label{thm:SGDLimDist}
	Consider the iterates $\{Y_k^{(\alpha)}\}$ generated by Algorithm (\ref{algo:Y:SGD}). Suppose that Assumptions \ref{as:noise} and \ref{as:SGD} are satisfied, then the following statements hold.
	\begin{enumerate}[(1)]
		\item There exists a threshold $\Bar{\alpha}>0$ such that for all $\alpha\in (0,\Bar{\alpha})$, the sequence of random variables $\{Y_k^{(\alpha)}\}$ converges weakly to some random variable $Y^{(\alpha)}$, which satisfies $\mathbb{E}[\|Y^{(\alpha)}\|_2^2]<\infty$.
		\item For any positive sequence $\{\alpha_i\}$ satisfying $\alpha_i\in (0,\Bar{\alpha})$ for all $i$ and $\lim_{i\rightarrow\infty}\alpha_i=0$, the sequence $\{Y^{(\alpha_i)}\}$ converges weakly to a random variable $Y$, which satisfies the following integral equation
		\begin{align}\label{eq:SGD}
			\mathbb{E}\left[\left(t^\top \Sigma t+2it^\top H_f Y\right)e^{i t^\top Y}\right]=0
		\end{align}
		for all $t\in\mathbb{R}^d$.
		In addition, suppose that Eq. (\ref{eq:SGD}) has a unique solution (in terms of the distribution of $Y$), then the distribution of $Y$ is a Multivariate normal distribution with mean zero and covariance matrix $\Sigma_Y$ being the unique solution of the Lyapunov equation $H_f\Sigma_Y+\Sigma_Y H_f^\top=\Sigma$.
	\end{enumerate}
\end{theorem}

\begin{remark}
	To establish Theorem \ref{thm:SGDLimDist} (2), we require Eq. (\ref{eq:SGD}) to have a unique solution in terms of the distribution of $Y$. Such assumption will be relaxed to some extent in Section \ref{subsec:uniqueness}.
\end{remark}

Since $\Sigma$ is positive definite (cf. Assumption \ref{as:noise}), and $H_f$ is also positive definite under Assumption \ref{as:SGD}, it is well established in the literature that the Lyapunov equation $H_f\Sigma_Y+\Sigma_Y H_f^\top=\Sigma$ has a unique solution, which is explicitly given by 
\begin{align*}
	\Sigma_Y=\int_{0}^\infty e^{-H_f u}\Sigma e^{-H_f^\top u}du.
\end{align*}

Consider the special case where $f(x)=x^2/2$. In this case we have $H_f=1$, and hence $\Sigma_Y=\frac{1}{2}\Sigma$ by the Lyapunov equation. As a result, the distribution of the limiting random variable $Y$ is Gaussian with mean zero, and covariance matrix being $\frac{1}{2}\Sigma$. This agrees with the illustrative example (which is for scalar valued iterates) presented in Section \ref{subsec:illustrative}.

From Theorem \ref{thm:SGDLimDist}, we see that the distribution of $Y$ only depends on the Hessian of $f(\cdot)$ at $x^*$. This makes intuitive sense because we are studying the asymptotic behavior of SA algorithm (\ref{algo:SGD}), and only the properties of $f(\cdot)$ around $x^*$ should play a role in determining the stationary distribution.

\subsection{Stochastic Approximation for Solving Linear Systems of Equations}\label{subsec:linear_SA}

Suppose that $F(x)=Ax+b$, where $A\in\mathbb{R}^{d\times d}$ and $b\in\mathbb{R}^d$. Then the SA algorithm (\ref{algo:sa}) becomes
\begin{align}\label{algo:linear_SA}
	X_{k+1}^{(\alpha)}=X_k^{(\alpha)}+\alpha\left(AX_k^{(\alpha)}+b+w_k\right),
\end{align}
which aims at iteratively solving the linear equation $Ax+b=0$.
Note that since the matrix $A$ is not necessarily symmetric, $F(x)$ need not be a gradient of any objective function. Such linear SA algorithm arises in many realistic applications. One typical example is the TD-learning algorithm for solving the policy evaluation problem in reinforcement learning, where the goal is to solve a linear Bellman equation. See \cite{bertsekas1996neuro,srikant2019finite} for more details.

To study the asymptotic behavior of Algorithm (\ref{algo:linear_SA}), we make the following assumption regarding the matrix $A$.

\begin{assumption}\label{as:linear_sa}
	The matrix $A$ is Hurwitz, i.e., all the eigenvalues of the matrix $A$ have strictly negative real part.
\end{assumption}
\begin{remark}
	Since $A$ being Hurwitz implies $A$ being non-singular, Assumption \ref{as:linear_sa} implies that the target equation $Ax+b=0$ has a unique solution, which we denote by $x^*$.
\end{remark}

Assumption \ref{as:linear_sa} is standard in studying linear SA algorithms. In particular, it was shown in the literature that under Assumption \ref{as:linear_sa} and some mild conditions on the noise $\{w_k\}$, Algorithm (\ref{algo:linear_SA}) converges in the mean square sense to a neighborhood around $x^*$ \citep{bertsekas1996neuro}.

To study the asymptotic distribution, for a fixed stepsize $\alpha$, we define the centered scaled iterate $Y_k^{(\alpha)}$ by $Y_k^{(\alpha)}=(X_k^{(\alpha)}-x^*)/\sqrt{\alpha}$ for all $k\geq  0$. To find the corresponding update equation for $Y_k^{(\alpha)}$, using the update equation for $X_k^{(\alpha)}$ and the fact that $Ax^*+b=0$, we obtain
\begin{align}\label{algo:linear_SA_scaled}
	Y_{k+1}^{(\alpha)}=(I+\alpha A) Y_k^{(\alpha)}+\sqrt{\alpha} w_k.
\end{align}
The full characterization of the two-step weak convergence (cf. Eq. (\ref{eq:two_limit})) of $\{Y_k^{(\alpha)}\}$ is presented in the following.

\begin{theorem}\label{thm:linear_SA}
	Consider the iterates $\{Y_k^{(\alpha)}\}$ generated by Algorithm (\ref{algo:linear_SA_scaled}). Suppose that Assumptions \ref{as:noise} and \ref{as:linear_sa} are satisfied, then the following statements hold.
	\begin{enumerate}[(1)]
		\item There exists a threshold $\Bar{\alpha}'>0$ such that for all $\alpha\in (0,\Bar{\alpha}')$, the sequence of random variables $\{Y_k^{(\alpha)}\}$ converges weakly to some random variable $Y^{(\alpha)}$, which satisfies $\mathbb{E}[\|Y^{(\alpha)}\|_2^2]<\infty$.
		\item For any positive sequence $\{\alpha_i\}$ satisfying $\alpha_i\in (0,\Bar{\alpha}')$ for all $i$ and $\lim_{i\rightarrow\infty}\alpha_i=0$, the sequence of random variables $\{Y^{(\alpha_i)}\}$ converges weakly to a random variable $Y$, which satisfies the following equation
		\begin{align}\label{eq:linearSA}
			\mathbb{E}\left[\left(t^\top \Sigma t-2it^\top AY\right)e^{it^\top Y}\right]=0,\quad \forall\;t\in\mathbb{R}^d.
		\end{align}
		In addition, suppose that Eq. (\ref{eq:linearSA}) has a unique solution, then the distribution of $Y$ is a Multivariate normal distribution with mean zero and covariance matrix being the unique solution of the Lyapunov equation $A \Sigma_Y+\Sigma_Y A^\top+\Sigma=0$.
	\end{enumerate}
\end{theorem}

Since the matrix $A$ is Hurwitz, and $\Sigma$ is positive definite, the existence and uniqueness of a positive definition solution to the Lyapunov equation $A \Sigma_Y+\Sigma_Y A^\top+\Sigma=0$ are guaranteed \citep{haddad2011nonlinear}. Lyapunov equations were used extensively in studying the stability of linear ordinary differential equations (ODE). Interestingly, it also shows up in determining the limit distribution of centered scaled iterates of discrete linear SA algorithms.

\subsection{Stochastic Approximation under Contraction Assumption}\label{subsec:contractive_SA}

Suppose that $F(x)=\mathcal{T}(x)-x$, where $\mathcal{T}:\mathbb{R}^d\times\mathbb{R}^d$ is a general nonlinear operator. In this case, Algorithm (\ref{algo:sa}) becomes
\begin{align}\label{algo:contractive_SA}
	X_{k+1}^{(\alpha)}=X_k^{(\alpha)}+\alpha\left(\mathcal{T}\left(X_k^{(\alpha)}\right)-X_k^{(\alpha)}+w_k\right),
\end{align}
which can be interpreted as an SA algorithm for finding the fixed-point of the operator $\mathcal{T}(\cdot)$. These type of algorithms arise in the context of reinforcement learning, where $\mathcal{T}(\cdot)$ is the Bellman operator. To proceed, we need the following definition.

\begin{definition}
	Let $\nu_i$, $1\leq i\leq d$ be positive real numbers. Then the weighted $\ell_2$ norm $\|\cdot\|_{\nu}$ with weights $\{\nu_i\}_{1\leq i\leq d}$ is defined by $\|x\|_{\nu}=(\sum_{i=1}^d\nu_ix_i^2)^{1/2}$ for all $x\in\mathbb{R}^d$.
\end{definition}

Next, we state our assumption regarding the operator $\mathcal{T}(\cdot)$. 

\begin{assumption}\label{as:contraction}
	The operator $\mathcal{T}(\cdot)$ is differentiable, and 
	there exists $\gamma\in (0,1)$ such that $\|\mathcal{T}(x_1)-\mathcal{T}(x_2)\|_\mu\leq \gamma\|x_1-x_2\|_\mu$ for any $x_1,x_2\in\mathbb{R}^d$, where $\|\cdot\|_\mu$ is some weighted $\ell_2$-norm with weights $\{\mu_i\}_{1\leq i\leq d}$.
\end{assumption}

Assumption \ref{as:contraction} essentially states that the operator $\mathcal{T}(\cdot)$ is a contraction mapping with respect to the weighted $\ell_2$-norm $\|\cdot\|_\mu$. By Banach fixed-point theorem \citep{banach1922operations}, the operator $\mathcal{T}(\cdot)$ has a unique fixed-point, which we denote by $x^*$.

We next write down the update equation of the centered scaled iterate $Y_k^{(\alpha)}=(X_k^{(\alpha)}-x^*)/\sqrt{\alpha}$ in the following:
\begin{align}\label{algo:Y_contractive_SA}
	Y_{k+1}^{(\alpha)}=Y_k^{(\alpha)}+\sqrt{\alpha}\left(\mathcal{T}\left(\sqrt{\alpha} Y_k^{(\alpha)}+x^*\right)-\left(\sqrt{\alpha} Y_k^{(\alpha)}+x^*\right)\right)+\sqrt{\alpha} w_k.
\end{align}
In the next theorem, we characterize the distribution of the limiting random vector $Y$ (cf. Eq. (\ref{eq:two_limit})). Let $J\in\mathbb{R}^{d\times d}$ be the Jacobian matrix of $\mathcal{T}(\cdot)$ evaluated at $x^*$.

\begin{theorem}\label{thm:contraction}
	Consider the iterates $\{Y_k^{(\alpha)}\}$ generated by Algorithm (\ref{algo:Y_contractive_SA}). Suppose that Assumptions \ref{as:noise} and \ref{as:contraction} are satisfied, then the following statements hold.
	\begin{enumerate}[(1)]
		\item There exists a threshold $\Bar{\alpha}''>0$ such that for all $\alpha\in (0,\Bar{\alpha}'')$, the sequence of random variables $\{Y_k^{(\alpha)}\}$ converges weakly to some random variable $Y^{(\alpha)}$, which satisfies $\mathbb{E}[\|Y^{(\alpha)}\|_2^2]<\infty$.
		\item For any positive sequence $\{\alpha_i\}$ satisfying $\alpha_i\in (0,\Bar{\alpha}'')$ for all $i$ and $\lim_{i\rightarrow\infty}\alpha_i=0$, the sequence of random variables $\{Y^{(\alpha_i)}\}$ converges weakly to a random variable $Y$, which satisfies the following equation
		\begin{align}\label{eq:contractive}
			\mathbb{E}\left[\left(t^\top \Sigma t-2it^\top (J-I)Y\right)e^{it^\top Y}\right]=0,\quad \forall\;t\in\mathbb{R}^d.
		\end{align}
		In addition, suppose that Eq. (\ref{eq:contractive}) has a unique solution, then the distribution of $Y$ is a Multivariate normal distribution with mean zero and covariance matrix being the unique solution of the Lyapunov equation $(J-I) \Sigma_Y+\Sigma_Y (J-I)^\top+\Sigma=0$.
	\end{enumerate}
\end{theorem}

Under the contraction assumption, each eigenvalue of the matrix $J$ is contained in the open unit ball on the complex plane. Therefore, the matrix $J-I$ is Hurwitz and hence the Lyapunov equation $(J-I) \Sigma_Y+\Sigma_Y (J-I)^\top+\Sigma=0$ has a unique positive definite solution $\Sigma_Y$ \citep{khalil2002nonlinear}.

\subsection{Regarding the Uniqueness Assumption}\label{subsec:uniqueness}

In Theorems \ref{thm:SGDLimDist}, \ref{thm:linear_SA}, and \ref{thm:contraction}, after obtaining the corresponding integral equation (e.g., Eqs. (\ref{eq:SGD}), (\ref{eq:linearSA}), and (\ref{eq:contractive})), to conclude that the distribution of $Y$ is Gaussian, we need to assume  that the equation has a unique solution. In this section, we show that such uniqueness assumption can be relaxed to some extend.

\subsubsection{Uni-Dimensional Setting}
Suppose that we are in the uni-dimensional setting, i.e., $d=1$. Then Eqs. (\ref{eq:SGD}), (\ref{eq:linearSA}), and (\ref{eq:contractive}) all reduce to an equation of the following form: $\mathbb{E}[(a t+2biY)e^{itY}]=0$ for all $t\in\mathbb{R}$, where $a$ and $b$ are positive constants. Let $\phi_Y(t)=\mathbb{E}[e^{itY}]$ be the characteristic function of $Y$. Then we can rewrite the previous equation as
\begin{align}\label{eq:81}
	a t \phi_Y(t)+2b\frac{d\phi_Y(t)}{dt}=0,
\end{align}
where the interchange of integral and differentiation is justified \citep{flanders1973differentiation}. Now Eq. (\ref{eq:81}) is an ODE, which has solutions of the form
\begin{align*}
	\phi_Y(t)=C\exp\left(-\frac{a}{4b}t^2\right),
\end{align*}
where $C$ is a constant. Since $\phi_Y(t)$ as a characteristic function satisfies $\phi_Y(0)=1$, we have $C=1$ and hence $\phi_Y(t)=\exp(-\frac{a}{4b}t^2)$, which is characteristic function for a Gaussian random variable with mean zero and covariance $\sqrt{a/(2b)}$. Therefore, the uniqueness assumption can be removed in the uni-dimensional setting.

\subsubsection{Multi-Dimensional Setting}
Moving to the multi-dimensional setting, consider Eq. (\ref{eq:SGD}) of Theorem \ref{thm:SGDLimDist} as a representative example. To show the same result of Theorem \ref{thm:SGDLimDist} (2) without imposing the uniqueness assumption, we consider the case where (1) the Hessian matrix $H_f$ of the objective function $f(\cdot)$ evaluated at $x^*$ is the identity matrix, and (2) the covariance matrix of the noise $w_k$ is also an identity matrix. Extending the result to the more general setting where $H_f$ and $\Sigma$ can be any positive definite matrices is a future research direction.

Similarly let $\phi_Y(t)=\mathbb{E}[e^{it^\top Y}]$ be the characteristic function of the random vector $Y$. Then in this case Eq. (\ref{eq:SGD}) becomes $t^\top  t\phi_Y(t)+2t^\top \nabla \phi_Y(t)=0$, which is equivalent to
\begin{align}\label{eq:120} 
	0&=t^\top  t+2t^\top \frac{\nabla \phi_Y(t)}{\phi_Y(t)}\nonumber\\
	&=t^\top  t+2t^\top \nabla \psi_Y(t),
\end{align}
where $\psi_Y(t):=\log(\phi_Y(t))$. To solve the partial differential equation (PDE) (\ref{eq:120}), we will first convert the PDE from Cartesian coordinates to spherical coordinates, which then is directly solvable.

The $d$-dimensional  spherical coordinate system consists of a radial coordinate $\rho$, and $d-1$ angular coordinates $\{\theta_i\}_{1\leq i\leq d-1}$. The relation between the Cartesian coordinates $(t_1,\cdots,t_d)$ and the  spherical coordinates $(\rho,\theta_1,\cdots,\theta_{d-1})$ is given by
\begin{align*}
	t_1&=\rho cos(\theta_1)\\
	t_2&=\rho sin(\theta_1)cos(\theta_2)\\
	t_3&=\rho sin(\theta_1)sin(\theta_2)cos(\theta_3)\\
	&\vdots\\
	t_{d-1}&=\rho sin(\theta_1)sin(\theta_2)\cdots sin(\theta_{d-2})cos(\theta_{d-1})\\
	t_d&=\rho sin(\theta_1)sin(\theta_2)\cdots sin(\theta_{d-2})sin(\theta_{d-1}).
\end{align*}
For simplicity of notation, denote $S\in\mathbb{R}^d$ as the spherical coordinate representation of $(t_1,\cdots,t_d)$ given above, i.e., $S_1=\rho cos(\theta_1)$, $\cdots$, $S_d=\rho sin(\theta_1)sin(\theta_2)\cdots sin(\theta_{d-2})sin(\theta_{d-1})$.

To proceed, we first compute the Jacobian matrix $J_d$ of the transformation in the following:
{\small \begin{align*}
	    J_d=\begin{bmatrix}
		    cos(\theta_1)& -\rho sin(\theta_1) & 0 &0 &\cdots &0\\
		    sin(\theta_1)cos(\theta_2)& \rho cos(\theta_1)cos(\theta_2)&-\rho sin(\theta_1)sin(\theta_2)&0&\cdots &0\\
		    \vdots &\vdots &\vdots&&\ddots&\vdots\\
		    \prod_{i=1}^{d-2}sin(\theta_i)cos(\theta_{d-1})&\cdots&\cdots& & &-\rho \prod_{i=1}^{d-1}sin(\theta_i)\\
		    \prod_{i=1}^{d-1}sin(\theta_i)&\rho\prod_{i=1}^{d-2}cos(\theta_i)sin(\theta_{d-1})&\cdots& & &\rho \prod_{i=1}^{d-2}sin(\theta_i)cos(\theta_{d-1})\\
		    \end{bmatrix}.
\end{align*}}
Using spherical coordinate system, Eq. (\ref{eq:120}) is equivalent to
\begin{align}\label{eq:121}
	\rho^2+2S^\top J_d \nabla \psi_Y(\rho,\theta_1,\cdots,\theta_{d-1}).
\end{align}
By direct computation (where we use $sin^2(\theta)+cos^2(\theta)=1$ for any $\theta$), we have $S^\top J_d=(\rho,0,\cdots,0)$. As a result, Eq. (\ref{eq:121}) simplifies to
\begin{align}\label{eq:122}
	\rho+2\frac{\partial \psi_Y(\rho,\theta_1,\cdots,\theta_{d-1})}{\partial \rho}=0,
\end{align}
which implies that $\psi_Y(\rho,\theta_1,\cdots,\theta_{d-1})=-\frac{\rho^2}{4}+C(\theta_1,\cdots,\theta_{d-1})$. Using the initial condition that $\psi_Y(0,\theta_1,\cdots,\theta_{d-1})=\log(\phi_Y(0))=\log(1)=0$ for any $\theta_1,\cdots,\theta_{d-1}$, we see that $C(\theta_1,\cdots,\theta_{d-1})=0$ and hence
$\phi_Y(\rho,\theta_1,\cdots,\theta_{d-1})=\frac{\rho^2}{4}$. Therefore, we have that $\psi_Y(t)=-\frac{t^\top t}{4}$, which implies $\phi_Y(t)=\exp(-\frac{t^\top t}{4})$. Therefore, the distribution of $Y$ is a multinormal distribution with mean zero and covariance matrix being $I_d/\sqrt{2}$. This agrees with Theorem \ref{thm:SGDLimDist} (2) when $H_f=\Sigma=I$, but the uniqueness assumption is not required to establish the result.

\subsection{Proof of Theorem \ref{thm:SGDLimDist}}

In this section, we present our proof for Theorem \ref{thm:SGDLimDist}. The proofs for Theorems \ref{thm:linear_SA} and \ref{thm:contraction} are similar and hence are omitted. 

Before going into details, we first highlight the main ideas for the proof. For Theorem \ref{thm:SGDLimDist} (1), to show that $\{Y^{(\alpha_i)}\}$ converges weakly to a some random variable $Y$, we show that any sub-sequence of $\{Y^{(\alpha_i)}\}$ further contains a weakly convergent sub-sequence, with a common weak limit. We do it by first showing that the sequence of random variables $\{Y^{(\alpha_i)}\}$ is tight. In particular, we show that there exists a constant $B$ independent of $\alpha$ such that $\mathbb{E}[\|Y^{(\alpha)}\|_2^2]\leq B$ for all small enough $\alpha$. This result in conjunction with the Markov inequality implies the tightness of $\{Y^{(\alpha_i)}\}$. As a result of tightness, $\{Y^{(\alpha_i)}\}$ contains a weakly convergent sub-sequence.  

Now consider Theorem \ref{thm:SGDLimDist} (2). For any positive sequence $\{\alpha_i\}$ such that $\lim_{i\rightarrow\infty}\alpha_i=0$, since the family of random variables $\{Y^{(\alpha_i)}\}$ is tight, there is a weakly convergent subsequence $\{Y^{\alpha_{i_k}}\}$. We further show that the weak limit $Y$ of the subsequence $\{Y^{(\alpha_{i_k})}\}$ solves Eq. (\ref{eq:SGD}). In this case, under the assumption that Eq. (\ref{eq:SGD}) has a unique solution, the random variable $Y$ is a Gaussian random variable with mean zero, and covariance matrix $\Sigma_Y$ being the unique solution of the Lyapunov equation $H_f \Sigma_Y+\Sigma_Y H_f^\top =\Sigma$. Since for every sequence $\{Y^{(\alpha_i)}\}$, there is a weakly convergent subsequence $\{Y^{(\alpha_{i_k})}\}$ with a common weak limit, the sequence of random variables $\{Y^{(\alpha_i)}\}$ also converge weakly to the same random variable $Y$.

\subsubsection{Proof of Theorem \ref{thm:SGDLimDist} (1)}
To prove the result, we will apply Proposition 2.1 in \cite{yu2020analysis}. For completeness, we first state this proposition using our notation in the following.

\begin{proposition}\label{prop:SGD}
	Consider $\{X_k^{(\alpha)}\}$ generated by Algorithm (\ref{algo:SGD}). Suppose that
	\begin{enumerate}[(a)]
		\item There exists $L'>0$ such that $\|\nabla f(x)\|_2\leq L'(1+\|x\|_2)$ for any $x\in\mathbb{R}^d$.
		\item There exist $\ell_1,\ell_2>0$ such that $\langle x, \nabla f(x)\rangle\geq \ell_1\|x\|_2^2-\ell_2$ for all $x\in\mathbb{R}^d$.
		\item The noise sequence $\{w_k\}$ is an i.i.d. sequence satisfying $\mathbb{E}[w_k]=0$ and $\mathbb{E}^{1/2}[\|w_k\|_2^2]\leq L''(1+\|x\|_2)$ for all $k\geq 0$, where $L''>0$ is a constant.
	\end{enumerate}
	Then, when the constant stepsize $\alpha<\frac{\ell_1-\sqrt{\max(\ell_1^2-(3L'^2+L''),0)}}{3L'^2+L''^2}$, the following statements hold.
	\begin{enumerate}[(1)]
		\item The iterates $\{X_k^{(\alpha)}\}$ admit a unique stationary distribution $\pi_\alpha$, which depends on the choice of $\alpha$. In addition, let $X^{(\alpha)}\sim \pi_\alpha$, then we have $\mathbb{E}[\|X^{(\alpha)}\|_2^2]<\infty$.
		\item For a test function $\phi:\mathbb{R}^d\mapsto\mathbb{R}$ satisfying $|\phi(x)|\leq L_\phi(1+\|x\|_2)$ for all $x\in\mathbb{R}^d$ and some $L_\phi>0$, and for any initialization $X_0^{(\alpha)}\in\mathbb{R}^d$ of the SGD algorithm (\ref{algo:SGD}), there exists $\rho\in (0,1)$ and $\kappa$
		(both depending on $\alpha$) such that we have $|\mathbb{E}[\phi(X_k^{(\alpha)})]-\pi_{\alpha}(\phi)|\leq \kappa \rho^k(1+\|X_0^{(\alpha)}\|_2^2)$, where $\pi_\alpha(\phi)=\mathbb{E}[\phi(X^{(\alpha)})]$.
	\end{enumerate}
\end{proposition}

Note that Proposition \ref{prop:SGD} (2) implies that $\{X_k^{(\alpha)}\}$ converges weakly to $X^{(\alpha)}$. To apply Proposition \ref{prop:SGD}, we next verify the assumptions.
\begin{enumerate}[(a)]
	\item Since the objective function $f(\cdot)$ is assumed to be $L$-smooth, we have for any $x\in\mathbb{R}^d$ that $\|\nabla f(x)-\nabla f(0)\|_2\leq L\|x\|_2$, which implies
	\begin{align*}
		\|\nabla f(x)\|_2\leq \|\nabla f(0)\|_2+L\|x\|_2\leq \underbrace{\max(\|\nabla f(0)\|_2,L)}_{L'}(\|x\|_2+1).
	\end{align*}
	\item Since the objective function is assumed to be $\sigma$-strongly convex, we have for any $x\in\mathbb{R}^d$:
	\begin{align*}
		f(0)-f(x)\geq \langle \nabla f(x),-x\rangle+\frac{\sigma}{2}\|x\|_2^2,
	\end{align*}
	which implies that
	\begin{align*}
		\langle \nabla f(x),x\rangle\geq \frac{\sigma}{2}\|x\|_2^2+f(x)-f(0)\geq \underbrace{\frac{\sigma}{2}}_{\ell_1}\|x\|_2^2+\underbrace{f(x^*)-f(0)-1}_{\ell_2}.
	\end{align*}
	\item This is immediately implied by Assumption \ref{as:noise}, with $L''=\text{Trace}(\Sigma)^{1/2}$.
\end{enumerate}

Now apply  Proposition \ref{prop:SGD}, when the stepsize $\alpha$ satisfies $\alpha<\frac{\sigma}{2(3L'^2+\text{Trace}(\Sigma))}$,
the SGD iterates $\{X^{(\alpha)}_k\}$ converge weakly to some random variable $X^{(\alpha)}$, which is distributed according to the unique stationary distribution $\pi_\alpha$. In addition, we have $\mathbb{E}[\|X^{(\alpha)}\|_2^2]<\infty$. Since $Y_k^{(\alpha)}$ is the centered scaled variant of $X_k^{(\alpha)}$, the sequence $\{Y_k^{(\alpha)}\}$ converges weakly to some random variable $Y^{(\alpha)}$ and $\mathbb{E}[\|Y^{(\alpha)}\|_2^2]<\infty$.

\subsubsection{Proof of Theorem \ref{thm:SGDLimDist} (2)}

Following the road map described in the beginning of this section, we present and prove a sequence of lemmas in the following. Together they imply the desired result.

\begin{lemma}\label{le:tightness}
	Let $\alpha_0=\sigma/L^2$. the family of random variables $\{Y^{(\alpha)}\}_{0<\alpha\leq \alpha_0}$ is tight.
\end{lemma}

\begin{proof}[Proof of Lemma \ref{le:tightness}]
	We first show that that there exists an absolute constant $C>0$ such that $\mathbb{E}[\|Y^{(\alpha)}\|^2]\leq C$ for any $\alpha\in (0,\alpha_0]$. Using the update equation (\ref{algo:sa}), we have
	\begin{align*}
		Y_{k+1}^{(\alpha)}&=Y_k^{(\alpha)}+\frac{\alpha}{g(\alpha)} \left(-\nabla f(Y_k^{(\alpha)}g(\alpha)+x^*)+w_k\right)\\
		&=Y_k^{(\alpha)}-\sqrt{\alpha}\nabla f(\sqrt{\alpha} Y_k^{(\alpha)}+x^*)+\sqrt{\alpha} w_k
	\end{align*}
	The existence and uniqueness of a stationary distribution $Y^{(\alpha)}$ is proved in Part (1) of this theorem. We next show that the family of random variables $\{Y^{(\alpha)}\}_{0\leq \alpha\leq \alpha_0}$ is tight. Using the equation
	\begin{align*}
		Y^{(\alpha)} \stackrel{D}{=}Y^{(\alpha)}-\sqrt{\alpha}\nabla f(\sqrt{\alpha} Y^{(\alpha)}+x^*)+\sqrt{\alpha} w,
	\end{align*}
	and we have
	\begin{align*}
		\mathbb{E}[\|Y^{(\alpha)}\|_2^2]=\;&\mathbb{E}[\|Y^{(\alpha)}\|_2^2]+\alpha\mathbb{E}\left[\left\|\nabla f(\sqrt{\alpha} Y^{(\alpha)}+x^*)\right\|_2^2\right]+\alpha\text{Trace}(\Sigma)\\
		&-2\sqrt{\alpha}\mathbb{E}\left[{Y^{(\alpha)}}^\top\nabla f(\sqrt{\alpha} Y^{(\alpha)}+x^*) \right].
	\end{align*}
	By smoothness, we have
	\begin{align*}
		\left\|\nabla f(\sqrt{\alpha} Y^{(\alpha)}+x^*)\right\|_2^2\leq L^2\alpha\|Y^{(\alpha)}\|_2^2.
	\end{align*}
	By strong convexity, we have
	\begin{align*}
		{Y^{(\alpha)}}^\top\nabla f(\sqrt{\alpha} Y^{(\alpha)}+x^*)&=\frac{1}{\sqrt{\alpha}}(\sqrt{\alpha}Y^{(\alpha)}+x^*-x^*)^\top\left(\nabla f(\sqrt{\alpha} Y^{(\alpha)}+x^*)-\nabla f(x^*)\right)\\
		&\geq \sigma\sqrt{\alpha}\|Y^{(\alpha)}\|_2^2.
	\end{align*}
	Therefore, we obtain
	\begin{align*}
		0\leq\;& L^2\alpha^2\mathbb{E}[\|Y^{(\alpha)}\|_2^2]+\alpha \text{Trace}(\Sigma)-2\sigma\alpha\mathbb{E}[\|Y^{(\alpha)}\|_2^2].
	\end{align*}
	When $\alpha\in (0,\alpha_0]$, we have from the previous inequality that
	\begin{align*}
		\mathbb{E}[\|Y^{(\alpha)}\|_2^2]\leq \frac{\text{Trace}(\Sigma)}{2\sigma-L^2\alpha}\leq \frac{\text{Trace}(\Sigma)}{\sigma}.
	\end{align*}
	Hence, for any $\alpha>0$, let $M=\sqrt{\text{Trace}(\Sigma)/\sigma\alpha}$, then we have
	\begin{align*}
		\mathbb{P}(\|Y^{(\alpha)}\|>M)\leq \frac{\mathbb{E}[\|Y^{(\alpha)}\|^2]}{M^2}\leq \frac{\text{Trace}(\Sigma)}{\sigma M^2}=\alpha
	\end{align*}
	for any $\alpha\in (0,\alpha_0]$. It follows that the family of random variables $\{Y^{(\alpha)}\}_{0<\alpha\leq \alpha_0}$ is tight.
\end{proof}

\begin{lemma}\label{le:equation}
	Let $\{\alpha_i\}$ be a positive sequence of real numbers such that $\lim_{i\rightarrow\infty}\alpha_i=0$. Suppose that $\{Y^{\alpha_i}\}$ converges weakly to some random variable $Y$. Then $Y$ satisfies the following equation 
	\begin{align}\label{eq:40}
		\mathbb{E}\left[\frac{t^\top \Sigma t}{2}  e^{i t^\top Y}\right]=-\mathbb{E}\left[\exp(it^\top Y)it^\top H_f Y\right].
	\end{align}
\end{lemma}

\begin{proof}[Proof of Lemma \ref{le:equation}]
	For any $i\geq 0$, we have 
	\begin{align*}
		Y^{(\alpha_i)} \overset{D}{=} Y^{(\alpha_i)}-\sqrt{\alpha_i}\nabla f(\sqrt{\alpha_i} Y^{(\alpha_i)}+x^*)+\sqrt{\alpha_i}w,
	\end{align*}
	which implies for any $t\in\mathbb{R}^d$:
	\begin{align}\label{eq:1}
		\mathbb{E}\left[e^{i t^\top Y^{(\alpha_i)}}\right]=\mathbb{E}\left[\exp\left(it^\top Y^{(\alpha_i)}\right)\exp\left(-\sqrt{\alpha_i}it^\top \nabla f(\sqrt{\alpha_i }Y^{(\alpha_i)}+x^*)\right)\right]\mathbb{E}\left[e^{\sqrt{\alpha_i} it^\top w}\right]
	\end{align}
	Using the Taylor's theorem and we have
	\begin{align*}
		&\exp\left(-\sqrt{\alpha_i}it^\top \nabla f(\sqrt{\alpha_i }Y^{(\alpha_i)}+x^*)\right)\\
		=\;&1-\sqrt{\alpha_i} it^\top \nabla f(\sqrt{\alpha_i}Y^{(\alpha_i)}+x^*)+\mathcal{O}\left(\alpha_i \|t\|^2\|\nabla f(\sqrt{\alpha_i}Y^{(\alpha_i)}+x^*\|^2\right).
	\end{align*}
	Using Theorem 3.3.20 from \cite{durrett2019probability} and we have
	\begin{align*}
		\mathbb{E}\left[e^{\sqrt{\alpha_i}it^\top Y^{(\alpha_i)}}\right]=1-\frac{\alpha_i t^\top \Sigma t}{2}+o(\alpha_i \|t\|^2).
	\end{align*}
	Using the previous two inequalities in Eq. (\ref{eq:1}) and we have
	\begin{align*}
		&\mathbb{E}\left[e^{i t^\top Y^{(\alpha_i)}}\right]\\
		=\;&\mathbb{E}\left[\exp(it^\top Y^{(\alpha_i)})\exp(-\sqrt{\alpha_i}it^\top \nabla f(\sqrt{\alpha_i }Y^{(\alpha_i)}+x^*))\right]\mathbb{E}[e^{\sqrt{\alpha_i} it^\top w}]\\
		=\;&\mathbb{E}\left[\exp(it^\top Y^{(\alpha_i)})\left(1-\sqrt{\alpha_i} it^\top \nabla f(\sqrt{\alpha_i}Y^{(\alpha_i)}+x^*)+\mathcal{O}\left(\alpha_i \|t\|^2\|\nabla f(\sqrt{\alpha_i}Y^{(\alpha_i)}+x^*\|^2\right)\right)\right]\times\\
		&\left(1-\frac{\alpha_i t^\top \Sigma t}{2}\right)
		+\mathbb{E}\left[\exp(it^\top Y^{(\alpha_i)})\exp(-\sqrt{\alpha_i}it^\top \nabla f(\sqrt{\alpha_i }Y^{(\alpha_i)}+x^*))\right]o(\alpha_i \|t\|^2)\\
		=\;&\mathbb{E}\left[e^{i t^\top Y^{(\alpha_i)}}\right]-\mathbb{E}\left[\frac{\alpha_i t^\top \Sigma t}{2}  e^{i t^\top Y^{(\alpha_i)}}\right]-\mathbb{E}\left[\exp(it^\top Y^{(\alpha_i)})\sqrt{\alpha_i} it^\top \nabla f(\sqrt{\alpha_i}Y^{(\alpha_i)}+x^*)\right]\\
		&+\mathbb{E}\left[\frac{\alpha_i t^\top \Sigma t}{2}\exp(it^\top Y^{(\alpha_i)})\sqrt{\alpha_i} it^\top \nabla f(\sqrt{\alpha_i}Y^{(\alpha_i)}+x^*)\right]\\
		&+\mathbb{E}\left[e^{i t^\top Y^{(\alpha_i)}}\mathcal{O}\left(\alpha_i \|t\|^2\|\nabla f(\sqrt{\alpha_i}Y^{(\alpha_i)}+x^*\|^2\right)\right]\\
		&-\mathbb{E}\left[\frac{\alpha_i t^\top \Sigma t}{2}e^{i t^\top Y^{(\alpha_i)}}\mathcal{O}\left(\alpha_i \|t\|^2\|\nabla f(\sqrt{\alpha_i}Y^{(\alpha_i)}+x^*\|^2\right)\right]\\
		&+\mathbb{E}\left[\exp(it^\top Y^{(\alpha_i)})\exp(-\sqrt{\alpha_i}it^\top \nabla f(\sqrt{\alpha_i }Y^{(\alpha_i)}+x^*))\right]o(\alpha_i \|t\|^2).
	\end{align*}
	Simplify the above equality and we obtain
	\begin{align*}
		\underbrace{\mathbb{E}\left[\frac{t^\top \Sigma t}{2}  e^{i t^\top Y^{(\alpha_i)}}\right]}_{T_1}=\;&-\underbrace{\mathbb{E}\left[\exp(it^\top Y^{(\alpha_i)})\frac{it^\top \nabla f(\sqrt{\alpha_i}Y^{(\alpha_i)}+x^*)}{\sqrt{\alpha_i}} \right]}_{T_2}\\
		&+\underbrace{\mathbb{E}\left[\frac{t^\top \Sigma t}{2}\exp(it^\top Y^{(\alpha_i)})\sqrt{\alpha_i} it^\top \nabla f(\sqrt{\alpha_i}Y^{(\alpha_i)}+x^*)\right]}_{T_3}\\
		&+\underbrace{\mathbb{E}\left[e^{i t^\top Y^{(\alpha_i)}}\mathcal{O}\left( \|t\|^2\|\nabla f(\sqrt{\alpha_i}Y^{(\alpha_i)}+x^*\|^2\right)\right]}_{T_4}\\
		&-\underbrace{\mathbb{E}\left[\frac{t^\top \Sigma t}{2}e^{i t^\top Y^{(\alpha_i)}}\mathcal{O}\left(\alpha_i \|t\|^2\|\nabla f(\sqrt{\alpha_i}Y^{(\alpha_i)}+x^*\|^2\right)\right]}_{T_5}\\
		&+\underbrace{\mathbb{E}\left[\exp(it^\top Y^{(\alpha_i)})\exp(-\sqrt{\alpha_i}it^\top \nabla f(\sqrt{\alpha_i }Y^{(\alpha_i)}+x^*))\right]\frac{o( \alpha_i \|t\|^2)}{\alpha_i}}_{T_6}.
	\end{align*}
	We next let $i$ go to infinity on both sides of the previous inequality and evaluate the limit of the terms $\{T_i\}_{1\leq i\leq 6}$. 
	
	Since $\{Y^{(\alpha_i)}\}$ converges weakly to some random variable $Y$, we have by continuity theorem (Theorem 3.3.17 in \cite{durrett2019probability}) that
	\begin{align*}
		\lim_{i\rightarrow \infty} \mathbb{E}\left[\frac{t^\top \Sigma t}{2}  e^{i t^\top Y^{(\alpha_{i})}}\right]=\frac{t^\top \Sigma t}{2}\mathbb{E}\left[  e^{i t^\top Y}\right].
	\end{align*}
	For the term $T_6$, we have by bounded convergence theorem that $\lim_{\alpha_i\rightarrow 0}T_6=0$. To evaluate the terms $T_2$ to  $T_5$, the following definition and result from \cite{van2000asymptotic} is needed.
	
	\begin{definition}
		A sequence of random variables $\{X_n\}$ is called asymptotically uniformly integrable if $\lim_{M\rightarrow\infty}\limsup_{n\rightarrow\infty}\mathbb{E}[|X_n|\mathbb{I}\{|X_n|>M\}]=0$.
	\end{definition}
	
	\begin{theorem}[Theorem 2.20 in \cite{van2000asymptotic}]\label{thm:weak_convergence}
		Let $f:\mathbb{R}^d\mapsto\mathbb{R}$ be measurable and continuous at every point in a set $C$. Let $X_n\Rightarrow X$, where $X$ takes its values in $C$. Then $\mathbb{E}[f (X_n)]\rightarrow\mathbb{E}[f(X)]$ if and only if the
		sequence of random variables $f(X_n)$ is asymptotically uniformly integrable.
	\end{theorem}
	
	Now consider the term $T_2$. Since
	\begin{align*}
		&\mathbb{E}\bigg[\;\left|\exp(it^\top Y^{(\alpha_i)})\frac{it^\top \nabla f(\sqrt{\alpha_i}Y^{(\alpha_i)}+x^*)}{\sqrt{\alpha_i}}\right|\times\\
		&\mathbb{I}\left\{\left|\exp(it^\top Y^{(\alpha_i)})\frac{it^\top \nabla f(\sqrt{\alpha_i}Y^{(\alpha_i)}+x^*)}{\sqrt{\alpha_i}}\right|>M\right\}\bigg]\\
		\leq \;&\frac{1}{M}\mathbb{E}\bigg[\;\left|\exp(it^\top Y^{(\alpha_i)})\frac{it^\top \nabla f(\sqrt{\alpha_i}Y^{(\alpha_i)}+x^*)}{\sqrt{\alpha_i}}\right|^2\times\\
		&\mathbb{I}\left\{\left|\exp(it^\top Y^{(\alpha_i)})\frac{it^\top \nabla f(\sqrt{\alpha_i}Y^{(\alpha_i)}+x^*)}{\sqrt{\alpha_i}}\right|>M\right\}\bigg]\\
		\leq \;&\frac{1}{\alpha_i M }\mathbb{E}\left[\;\left|t^\top \nabla f(\sqrt{\alpha_i}Y^{(\alpha_i)}+x^*)\right|^2\mathbb{I}\left\{\left|\exp(it^\top Y^{(\alpha_i)})\frac{it^\top \nabla f(\sqrt{\alpha_i}Y^{(\alpha_i)}+x^*)}{\sqrt{\alpha_i}}\right|>M\right\}\right]\\
		\leq \;&\frac{\|t\|^2}{\alpha_i M }\mathbb{E}\left[\;\|\nabla f(\sqrt{\alpha_i}Y^{(\alpha_i)}+x^*)\|^2\right]\\
		\leq \;&\frac{L\|t\|^2}{M }\mathbb{E}\left[\;\|Y^{(\alpha_i)}\|^2\right]\\
		\leq \;&\frac{L\|t\|^2}{M }\mathbb{E}\left[\;\|Y^{(\alpha_i)}\|^2\right]\\
		\leq \;&\frac{L\text{Trace}(\Sigma)\|t\|^2}{\sigma M },
	\end{align*}
	which goes to zero as $M\rightarrow\infty$, we have by Theorem \ref{thm:weak_convergence} that
	\begin{align*}
		\lim_{i\rightarrow \infty}T_2=\mathbb{E}\left[\exp(it^\top Y)it^\top H_f Y\right].
	\end{align*}
	Using the same line of analysis, we have $\lim_{i\rightarrow \infty}T_3=\lim_{i\rightarrow \infty}T_4=\lim_{i\rightarrow \infty}T_5=0$. It follows that
	\begin{align*}
		\mathbb{E}\left[\frac{t^\top \Sigma t}{2}  e^{i t^\top Y}\right]=-\mathbb{E}\left[\exp(it^\top Y)it^\top H_f Y\right].
	\end{align*}
	Rearranging terms and we obtain the resulting equation.
\end{proof} 

\begin{lemma}\label{le: unique solution}
	Suppose Eq. (\ref{eq:40}) admits a unique solution. Then the random variable $Y$ given in Lemma \ref{le:equation} follows a Gaussian distribution with mean zero, and covariance matrix $\Sigma_Y$ being the unique solution of the Lyapunov equation $H_f \Sigma_Y+\Sigma_Y H_f^\top =\Sigma$.
\end{lemma}
\begin{proof}[Proof of Lemma \ref{le: unique solution}]
	Suppose that Eq. (\ref{eq:SGD}) has a unique solution, we only need to verify that the multinormal distribution with mean zero and covariance matrix being the solution to the Lyapunov equation $H_f^\top \Sigma_Y+\Sigma_Y H_f=\Sigma$ solves equation \eqref{eq:SGD}.
	\begin{align*}
		&\mathbb{E}\left[(2it^\top H_f Y + t^\top \Sigma t)e^{it^\top Y}\right]\\
		=\;&C\int_{\mathbb{R}^d}(2it^\top H_f y + t^\top \Sigma t)e^{it^\top y}e^{-\frac{1}{2}y^\top \Sigma_Y^{-1} y}dy\tag{$C=\frac{1}{\sqrt{\left( 2\pi \right)^d \textit{det}(\Sigma_Y)}}$}\\
		=\;&Ce^{-\frac{1}{2}t^\top \Sigma_Y t}\int_{\mathbb{R}^d}(2it^\top H_f y + t^\top \Sigma t)e^{-\frac{1}{2}(y-i\Sigma_Y t)^\top \Sigma_Y^{-1}(y-i\Sigma_Y t)}dy\\
		=\;&Ce^{-\frac{1}{2}t^\top \Sigma_Y t}\int_{\mathbb{R}^d}(2it^\top H_f(z+i\Sigma_Y t)+t^\top \Sigma t)e^{-\frac{1}{2}z^\top \Sigma_Y^{-1} z} dz\tag{change of variable}\\
		=\;&Ce^{-\frac{1}{2}t^\top \Sigma_Y t}\int_{\mathbb{R}^d}(-2t^\top H_f\Sigma_Y t+t^\top \Sigma t)e^{-\frac{1}{2}z^\top \Sigma_Y^{-1} z} dz\\
		=\;&Ce^{-\frac{1}{2}t^\top \Sigma_Y t}\int_{\mathbb{R}^d}(-t^\top (H_f\Sigma_Y+\Sigma_Y H_f^\top) t-t^\top t)e^{-\frac{1}{2}z^\top \Sigma_Y^{-1} z} dz\\
		=\;&Ce^{-\frac{1}{2}t^\top \Sigma_Y t}\int_{\mathbb{R}^d}(-t^\top \Sigma t+t^\top \Sigma t)e^{-\frac{1}{2}z^\top \Sigma_\alpha z} dz\tag{The Lyapunov equation}\\
		=\;&0.
	\end{align*}
\end{proof}

\section{Identifying the Suitable Scaling Function for More General SA Algorithms}\label{sec:scaling}
In the previous section, we have shown that for several particular SA algorithms (e.g. SGD, linear SA, and contractive SA), the scaling function is $g(\alpha)=\sqrt{\alpha}$ and distribution of the limiting random variable $Y$ is Gaussian. In this section, we consider more general SA algorithms. We first show impirically in the following section that in general the scaling function need not be $g(\alpha)=\sqrt{\alpha}$, and the distribution of $Y$ need not be Gaussian.

\subsection{Numerical Experiments}\label{subsec:experiment}

Suppose that Algorithm  (\ref{algo:sa}) is the SGD algorithm for minimizing the scalar objective $f(x)=x^4/4$. That is:
\begin{align}\label{algo:sa:x^4}
	X_{k+1}^{(\alpha)}=X_k^{(\alpha)}+\alpha\left(-(X_k^{(\alpha)})^3+w_k\right).
\end{align}
Note that $f(\cdot)$ in this case is neither smooth nor strongly convex. It is clear that the unique minimizer of $f(\cdot)$ is zero. Let the centered scaled iterate $Y_k^{(\alpha)}$ be defined by $Y_k^{(\alpha)}=X_k^{(\alpha)}/g(\alpha)$. We next use numerical simulation to show that the correct scaling function in this case is $g(\alpha)=\alpha^{1/4}$ instead of $g(\alpha)=\sqrt{\alpha}$.

\begin{figure}[h]
	\begin{minipage}{0.45\textwidth}
		\centering
		\includegraphics[width=\linewidth]{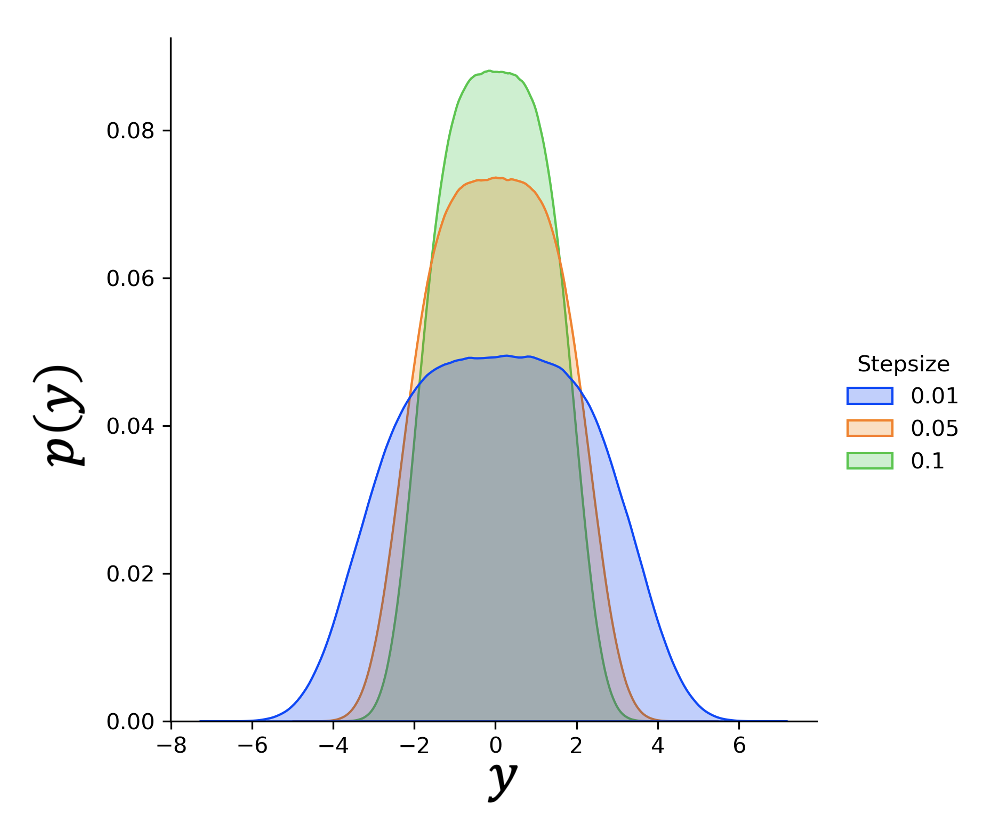}
		\caption{Estimated Density Functions when choosing $g(\alpha)=\alpha^{1/2}$}
		\label{fig:1}
	\end{minipage}
	\hspace{0.02\textwidth}
	\begin{minipage}{0.45\textwidth}
		\centering
		\includegraphics[width=\linewidth]{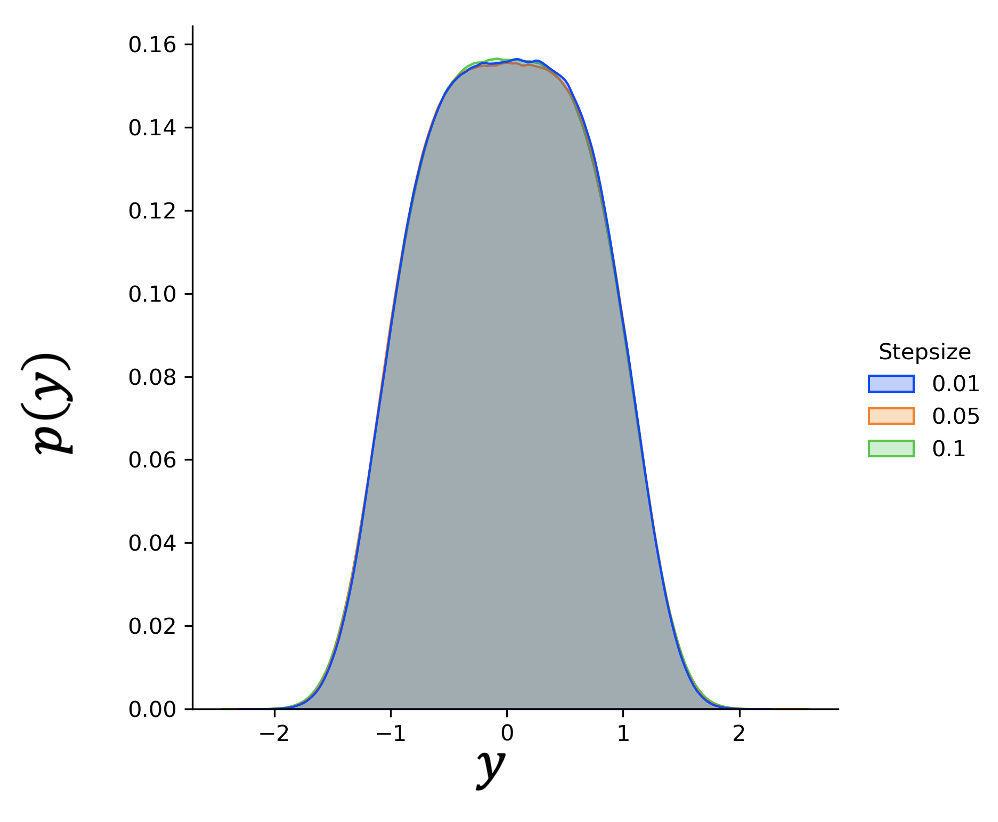}
		\caption{Estimated Density Functions when choosing $g(\alpha)=\alpha^{1/4}$}
		\label{fig:2}
	\end{minipage}
	\vfill
	\begin{minipage}{0.45\textwidth}
		\centering
		\includegraphics[width=\linewidth]{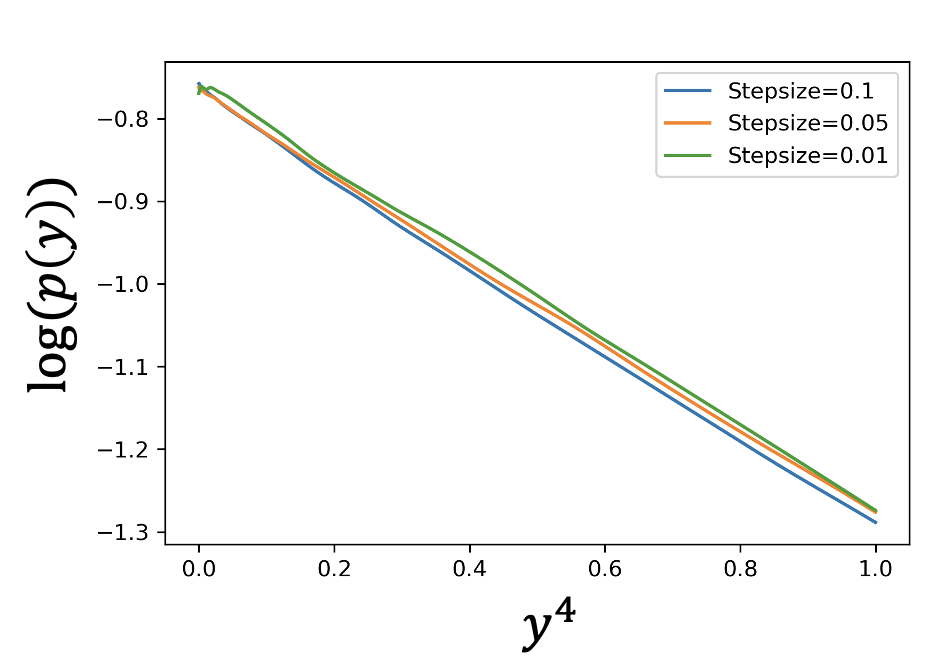}
		\caption{$\log(p_Y(y))$ as a function of $y^4$}\label{fig:3}
	\end{minipage}
\end{figure}

In Figures \ref{fig:1} and \ref{fig:2}, we plot the empirical density function of $Y^{(\alpha)}$ for different $\alpha$. For the right scaling function, we expect the density function to converge as $\alpha$ decreases, while for the wrong scaling function, we expect the density function to change drastically for order-wise different $\alpha$. As we see, it is clear that $g(\alpha)=\sqrt{\alpha}$ is not suitable in this case, and $g(\alpha)=\alpha^{1/4}$ seems to be the right scaling. 

To further verify this result, we plot the logarithmic empirical density function as a function of $y^4$ in Figure \ref{fig:3}. We observe linear growth in Figure \ref{fig:3}. This indicates that the density function $p_Y(y)$ is proportional to $e^{\beta y^4}$, where $\beta$ is some numerical constant. Therefore, numerical experiments suggest that the distribution of $Y$ is not Gaussian but super Gaussian in this problem.

\subsection{A Method to Determine the Suitable Scaling Function}

Inspired by the numerical simulations provided in the previous section, we here provide a method to determine the correct scaling function for general SA algorithms. 

To gain intuition, we consider the centered scaled iterates $Y_k^{(\alpha)}=X_k^{(\alpha)}/\alpha^{1/4}$ for SA algorithm (\ref{algo:sa:x^4}). The update equation of $Y_k^{(\alpha)}$ is given by
\begin{align*}
	Y_{k+1}^{(\alpha)}=Y_k^{(\alpha)}-\alpha^{3/2}(Y_k^{(\alpha)})^3+\alpha^{3/4}w_k.
\end{align*}
Notably, the factor in terms of the stepsize $\alpha$ in front of the term $(Y_k^{(\alpha)})^3$ is $\alpha^{3/2}$, which is equal to the square of the factor $\alpha^{3/4}$ in front of the noise term $w_k$. 

Now for general SA algorithm (\ref{algo:sa}), by rewriting the update equation (\ref{algo:sa}) in terms of the centered scaled iterate $Y_k^{(\alpha)}=(X_k^{(\alpha)}-x^*)/g(\alpha)$, we have
\begin{align}\label{eq:50}
	Y_{k+1}^{(\alpha)}
	=Y_k^{(\alpha)}+\left(\frac{\alpha}{g(\alpha)}\right)^2\frac{g(\alpha)F(Y_k^{(\alpha)}g(\alpha)+x^*)}{\alpha}+\frac{\alpha}{g(\alpha)}w_k.
\end{align}
In view of the previous equation and the empirical observations in the previous section, we see that we need to choose a scaling function $g(\alpha)$ such that the following condition is satisfied. 

\begin{condition}\label{con:1}
	The scaling function $g(\cdot)$ should be chosen such that
	\begin{enumerate}[(1)]
		\item $\lim_{\alpha\rightarrow 0}\frac{\alpha}{g(\alpha)}=0$ and $\lim_{\alpha\rightarrow 0}g(\alpha)=0$
		\item The function $\Tilde{F}:\mathbb{R}^d\mapsto\mathbb{R}^d$ defined by $\Tilde{F}(y)=\lim_{\alpha\rightarrow 0}\frac{g(\alpha)F(yg(\alpha)+x^*)}{\alpha}$
		is a nontrivial function in the sense that $\Tilde{F}(\cdot)$ is not identically equal to zero or infinity.
	\end{enumerate}
\end{condition}

We next verify the choice of scaling functions in Section \ref{sec:distribution} using our proposed Condition \ref{con:1}. For SGD with a smooth and strong convex objective, since
\begin{align*}
	\sigma\|x-x^*\|_2\leq  \|\nabla f(x)-\nabla f(x^*)\|_2=\|\nabla f(x)\|_2\leq L\|x-x^*\|_2,\quad \forall\;x\in\mathbb{R}^d,
\end{align*}
we have
\begin{align*}
	\sigma\frac{g(\alpha)^2}{\alpha}\|y\|_2\leq \left\|\frac{g(\alpha) \nabla f(g(\alpha) y+x^*)}{\alpha}\right\|_2\leq L\frac{g(\alpha)^2}{\alpha}\|y\|_2.
\end{align*}
In view of the previous inequality and Condition \ref{con:1}, it is clear that the only possible choice of $g(\alpha)$ is $g(\alpha)=\sqrt{\alpha}$. 

For linear SA algorithm studied in Section \ref{subsec:linear_SA}, one can also easily show using Condition \ref{con:1} that $g(\alpha)=\sqrt{\alpha}$. As for contractive SA studied in Section \ref{subsec:contractive_SA}, using the contraction property and we have 
\begin{align*}
	(1-\gamma)\|x-x^*\|_\mu\leq  \|\mathcal{T}(x)-x\|_\mu=\|\mathcal{T}(x)-\mathcal{T}(x^*)-(x-x^*)\|_\mu\leq (1+\gamma)\|x-x^*\|_\mu.
\end{align*}
It follows that
\begin{align*}
	\frac{g(\alpha)^2}{\alpha}(1-\gamma)\|y\|_\mu\leq \left\|\frac{g(\alpha) [\mathcal{T}(g(\alpha) y+x^*)-(g(\alpha) y+x^*)]}{\alpha}\right\|_\mu\leq \frac{g(\alpha)^2}{\alpha}(1+\gamma)\|y\|_\mu.
\end{align*}
Since all norms are ``equivalent'' in finite dimensional spaces, the previous inequality implies that we must choose $g(\alpha)=\sqrt{\alpha}$.

Now to further verify the correctness of the scaling function suggested by Condition \ref{con:1}, consider the SGD algorithm
\begin{align*}
	X_{k+1}^{(\alpha)}=X_k^{(\alpha)}+\alpha(-\nabla  f(X_k^{(\alpha)})+w_k)
\end{align*}
with the following two choices of objective functions: (1) $f(x)=e^{x^2}$, and (2) $f(x)=\frac{x^4}{4}+\frac{\sin^2(x)}{2}$. Note that in these two cases the function $f(\cdot)$ is not smooth and strongly convex. 

\textbf{Case 1.} In the first case where $f(x)=e^{x^2}$, since
\begin{align*}
	\left\|\frac{g(\alpha) F(yg(\alpha))}{\alpha}\right\|_2=\frac{g(\alpha)^2}{\alpha}2|y|e^{(yg(\alpha))^2},
\end{align*}
when choosing $g(\alpha)=\sqrt{\alpha}$, we have $\Tilde{F}(y)=\lim_{\alpha\rightarrow 0}\frac{g(\alpha)^2}{\alpha}2ye^{(yg(\alpha))^2}=2y$. 

One interesting implication of this example is the following. In the three SA algorithms considered in Section \ref{sec:distribution}, it seems that it is the function $F(\cdot)$ that appears in the algorithm determines the distribution of $Y$. However, the above example suggests that it is the function $\Tilde{F}(\cdot)$ of Condition \ref{con:1} instead of $F(\cdot)$ that directly impacts the distribution $Y$. In SGD with a smooth and strongly convex objective, linear SA, and contractive SA, $F(\cdot)$ and $\Tilde{F}(\cdot)$ happen to coincides, but this is in general not the case. In fact, since we have $\frac{de^{x^2}}{dx}=\sum_{i=1}^\infty (2i)x^{2i-1}$ by Taylor series, $\Tilde{F}(\cdot)$ in this example is exactly the dominant terms appears in the Taylor series. In addition, this suggests that the distribution of the limiting random variable $Y$ has a density function proportional to $e^{\beta' x^2}$, where $\beta'$ is a numerical constant.

We next verify this choice of $g(\alpha)$ and the distribution of $Y^{(\alpha)}$ for small enough $\alpha$ using numerical simulation in the following.

\begin{figure}[h]
	\begin{minipage}{0.45\textwidth}
		\centering
		\includegraphics[width=\linewidth]{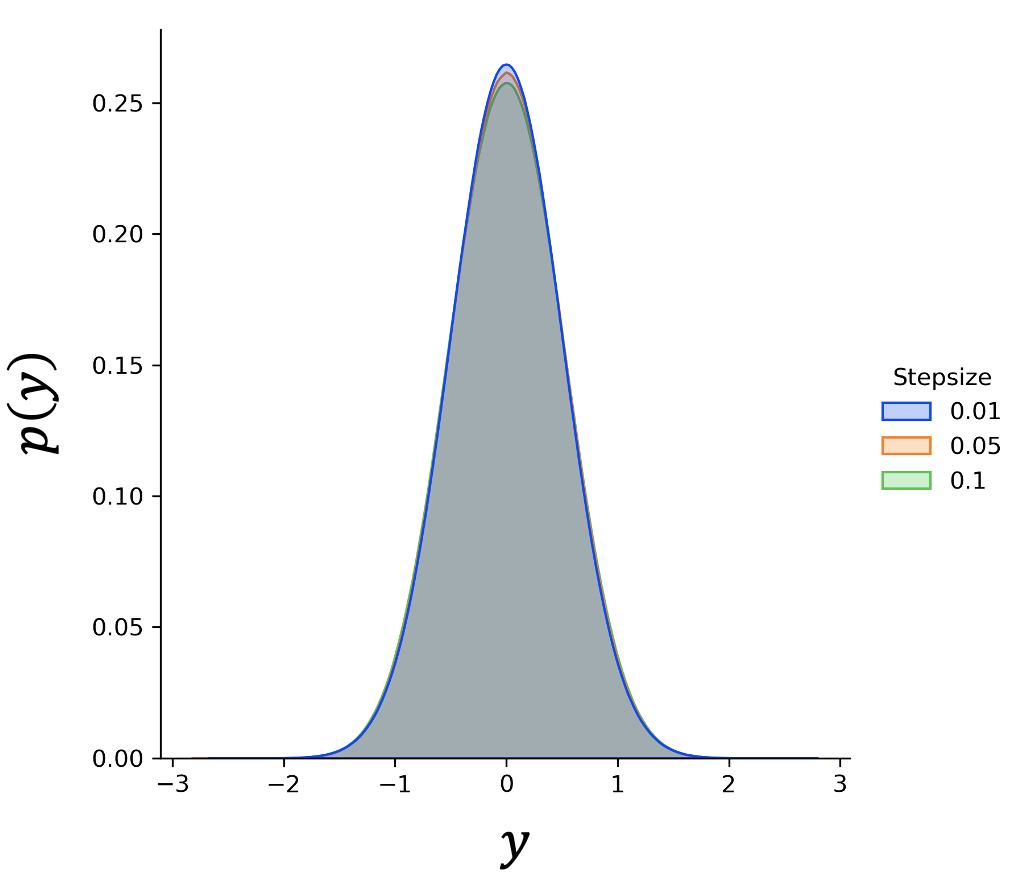}
		\caption{Estimated Density Functions when choosing $g(\alpha)=\alpha^{1/2}$}
		\label{fig:4}
	\end{minipage}
	\hfill
	\begin{minipage}{0.45\textwidth}
		\centering
		\includegraphics[width=\linewidth]{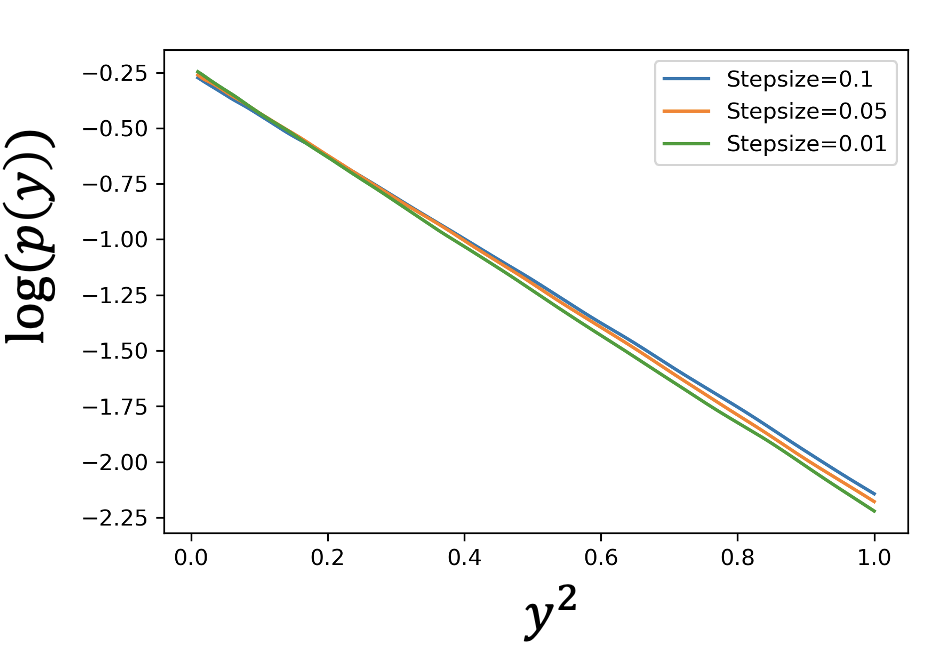}
		\caption{$\log(p_Y(y))$ as a function of $y^2$}\label{fig:5}
	\end{minipage}
\end{figure}

We see from Figure \ref{fig:4} that with the scaling function $g(\alpha)=\sqrt{\alpha}$, the empirical density function of the random variable $Y^{(\alpha)}$ seems to converge. and Figure \ref{fig:5} further justifies this result by showing that the density function $p_Y(y)$ of the distribution of $Y$ in this case is proportional to $e^{\beta' x^2}$, where $\beta'$ is a numerical constant.

\textbf{Case 2.} Now consider case where $f(x)=\frac{x^4}{4}+\frac{\sin^2(x)}{2}$. Observe that
\begin{align*}
	\left\|\frac{g(\alpha)F(yg(\alpha))}{\alpha}\right\|_2=\frac{g(\alpha)}{\alpha}|y^3g(\alpha)^3+sin(yg(\alpha))cos(yg(\alpha))|.
\end{align*}
Since $\lim_{x\rightarrow 0}\frac{sin(x)}{x}=1$, the only possible choice of the scaling function $g(\alpha)$ to satisfy Condition \ref{con:1} (2) is $g(\alpha)=\sqrt{\alpha}$. In this case, we have $\Tilde{F}(y)=\lim_{\alpha\rightarrow 0} \frac{1}{\sqrt{\alpha}}y^3\alpha^{3/2}+sin(y\sqrt{\alpha})cos(y\sqrt{\alpha})=y$ by L'Hôpital's rule. This is another example where $F(\cdot)\neq \Tilde{F}(\cdot)$. In fact, since $x^4$ is dominated by $sin^2(x)$ as $x$ approaches $x^*$ (which is $0$), the scaling function and the function $\Tilde{F}(\cdot)$ are determined only by the dominant term.
.

Similarly, we verify this choice of scaling function via numerical experiments. In Figures \ref{fig:10} and \ref{fig:11}, we plot the empirical density function of the random variable $Y^{(\alpha)}$ for different stepsize $\alpha$, and see if the density function converges as $\alpha$ goes to zero. The results suggest that $g(\alpha)=\alpha^{1/2}$ seems to be the correct scaling. To further verify this result, we plot the logarithmic function of the empirical density of $Y^{(\alpha)}$ as a function of $y^2$ and observe straight lines. Therefore, the distribution of $Y^{(\alpha)}$ is proportional to $e^{\beta'' x^2}$, where $\beta''$ is a numerical constants.

\begin{figure}[h]
	\begin{minipage}{0.45\textwidth}
		\centering
		\includegraphics[width=\linewidth]{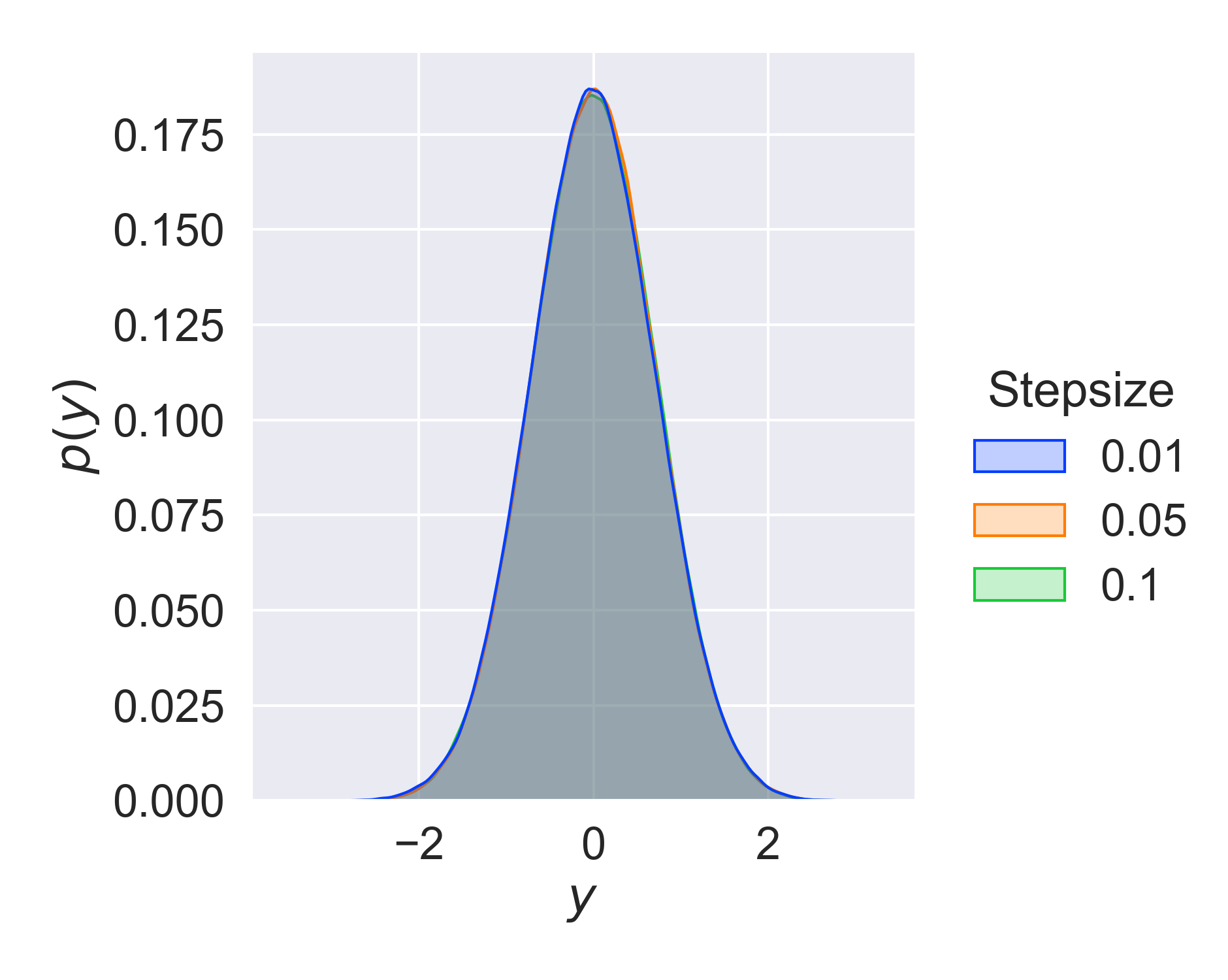}
		\caption{Estimated Density Functions when choosing $g(\alpha)=\alpha^{1/2}$}
		\label{fig:10}
	\end{minipage}
	\hspace{0.02\textwidth}
	\begin{minipage}{0.45\textwidth}
		\centering
		\includegraphics[width=\linewidth]{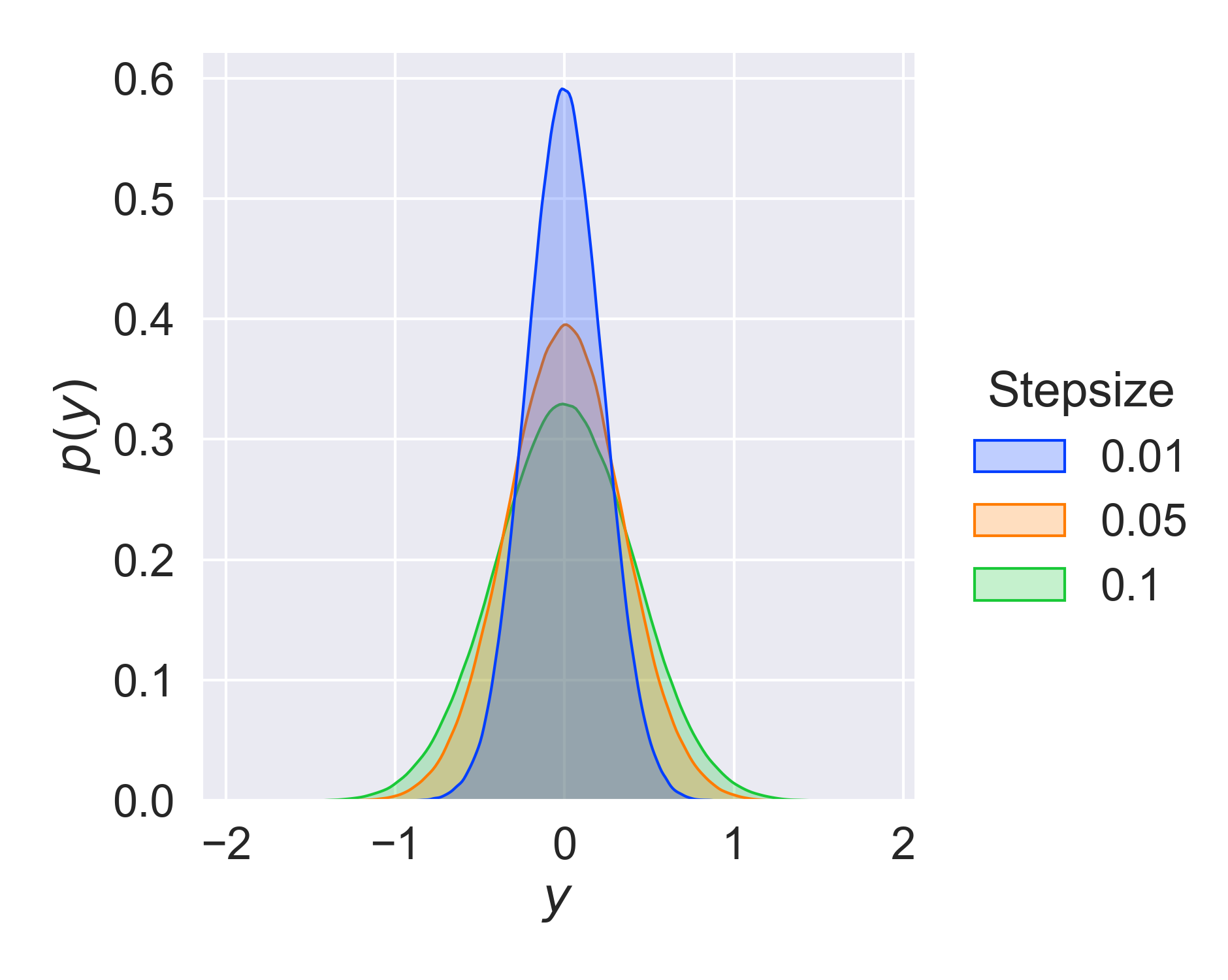}
		\caption{Estimated Density Functions when choosing $g(\alpha)=\alpha^{1/4}$}
		\label{fig:11}
	\end{minipage}
	\vfill
	\begin{minipage}{0.45\textwidth}
		\centering
		\includegraphics[width=\linewidth]{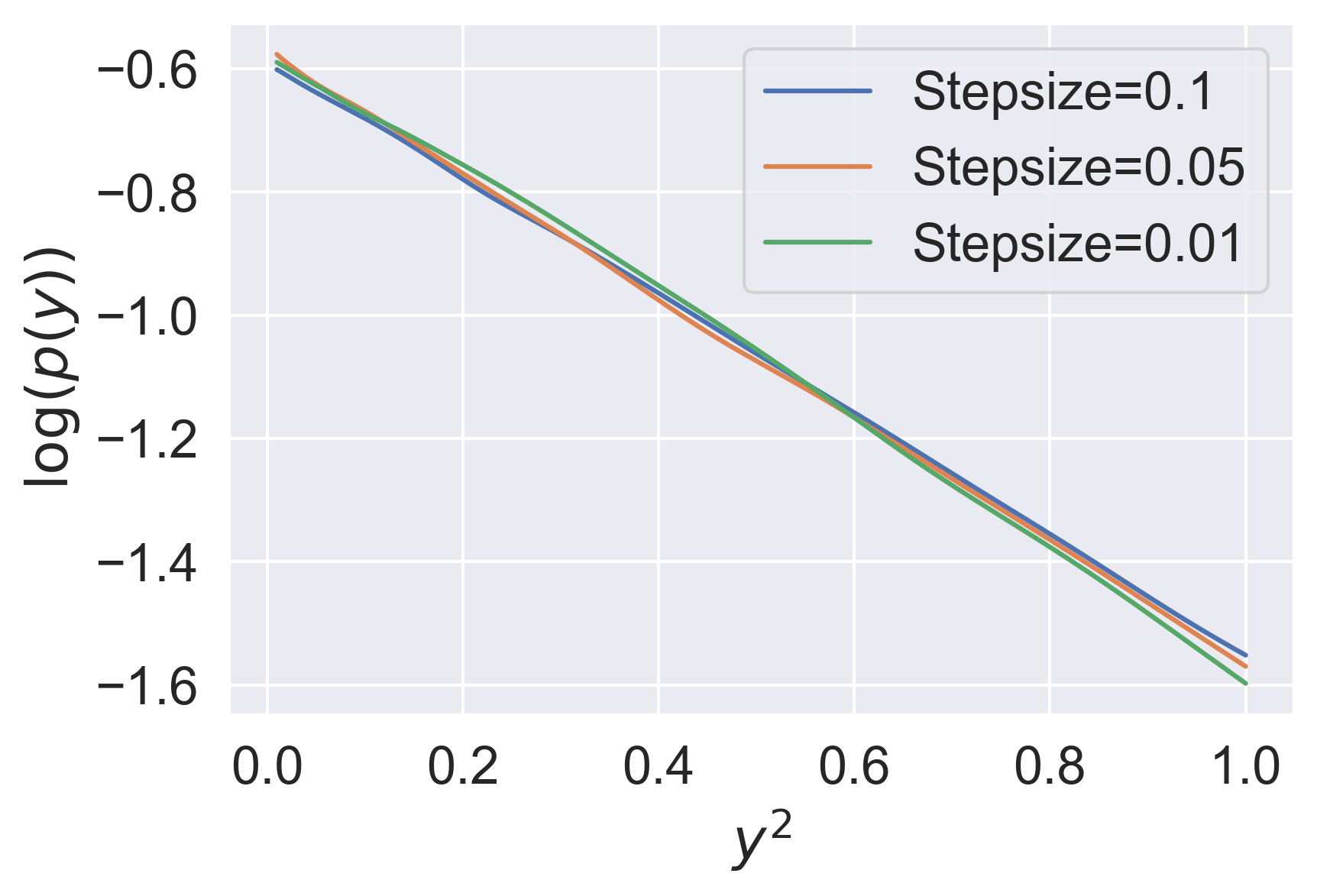}
		\caption{$\log(p_Y(y))$ as a function of $y^2$}
		\label{fig:12}
	\end{minipage}
\end{figure}

\subsection{Connection to Euler-Maruyama Discretization Scheme for Approximating SDE}

The choice of the scaling function suggested by Condition \ref{con:1} has an insightful connection to the Euler-Maruyama discretization scheme for approximate the solution of an SDE, as elaborated below. Let $(B_t)_{t\geq 0}$ be a Brownian motion. Consider the following SDE:
\begin{align}\label{eq:SDE}
	dX_t=F(X_t)dt+dB_t
\end{align}
with initial condition $X_0$. The Euler-Maruyama discretization $\{\hat{X}_k\}$ to the solution $(X_t)$ of SDE (\ref{eq:SDE}) is defined as follows. Let $\Delta t$ be the discretization accuracy. Set $\hat{X}_0=X_0$, and recursively define $\hat{X}_k$ according to 
\begin{align*}
	\hat{X}_{k+1}=\hat{X}_k+\Delta t F(\hat{X}_k)+ (B_{(k+1)\Delta t}-B_{k \Delta t}).
\end{align*}
Since $(B_t)_{t\geq 0}$ is a Brownian motion, we have $(B_{(k+1)\Delta t}-B_{k \Delta t})\sim \mathcal{N}(0,\Delta t)$. Therefore, by letting $\{Z_k\}$ be an i.i.d. sequence of standard normal random variables, we can rewrite the previous equation as
\begin{align}\label{eq:Euler-discretization}
	\hat{X}_{k+1}=\hat{X}_k+\Delta t F(\hat{X}_k)+ \sqrt{\Delta t} Z_k.
\end{align}

The approximation property of the Euler-Maruyama discretization (\ref{eq:Euler-discretization}) to its corresponding SDE (\ref{eq:SDE}) has been studied in the literature, see \cite{mou2019improved}. Specifically, it was shown that under some mild conditions on $F(\cdot)$, the Euler-Maruyama scheme is known to have the first-order accuracy of the SDE (\ref{eq:SDE}). As a consequence, intuitively, when $(X_t)_{t\geq 0}$ has a stationary distribution $\mu$, the limiting distribution $\mu_{\Delta t}$ of $\{\hat{X}_k\}$ as a function of the discretization accuracy $\Delta t$ should converge weakly to $\mu$ as $\Delta t$ tends to zero. If we view the discretization accuracy $\Delta t$ as the stepsize in Eq. (\ref{eq:Euler-discretization}). In order for $\mu_{\Delta t}$ to converge to some nontrivial distribution $\mu$ as $\Delta t$ tends to zero, it is important to notice that the scaling factor of the noise $Z_k$ in terms of $\Delta t$ must be order-wise equal to the square root of the scaling factor of $F(\hat{X}_k)$. This observation coincides with Eq. (\ref{eq:50}) in the previous section, which eventually leads to our Condition \ref{con:1}.

\section{Conclusion}\label{sec:conclusion}

In this paper, we characterize the asymptotic stationary distribution of properly centered scaled iterate of SA algorithms. In particular, we show that for (1) SGD with smooth and strongly convex objective, (2) linear SA, and (3) contractive SA, the scaling function is $g(\alpha)=\sqrt{\alpha}$ and the corresponding stationary distribution is Gaussian. For SA beyond these cases, we empirical show that the stationary distribution need not be Gaussian, and provide a method for determine the suitable scaling function. Since our paper is the first study for this problem, it might open a door for research in this direction.

\bibliographystyle{imsart-nameyear}
\bibliography{references}

\end{document}